\documentclass{article}

\PassOptionsToPackage{numbers, compress}{natbib}

\usepackage[final]{neurips_2022}

\usepackage[utf8]{inputenc} 
\usepackage[T1]{fontenc}    
\usepackage[hidelinks]{hyperref}       
\usepackage{url}            
\usepackage{booktabs}       
\usepackage{amsfonts}       
\usepackage{nicefrac}       
\usepackage{microtype}      
\usepackage{xcolor}         
\usepackage{dsfont}
\usepackage{comment}
\usepackage{rotating}
\usepackage[toc,page,header]{appendix}
\usepackage{minitoc}

\usepackage{pgfplots, pgfplotstable}
\usepgfplotslibrary{fillbetween}
\usepackage{graphics}
\usepackage{bm}

\pgfplotsset{every axis/.append style={
                title style={yshift=-0.2cm},
                label style={font=\footnotesize},
                ylabel style={yshift=2cm},
                tick label style={font=\scriptsize},
                legend style={font=\scriptsize}}}
\pgfplotsset{compat=1.16}
\pgfplotsset{every tick label/.append style={font=\tiny\rmfamily}}
\pgfplotsset{every axis/.append style={font=\scriptsize\rmfamily, ylabel near ticks, xlabel near ticks}}
\usepgfplotslibrary{groupplots}
\usetikzlibrary{external,matrix}

\usepackage{utils}
\usepackage{multicol}

\definecolor{grgreen}{rgb}{0.0, 0.5, 0.0}
\newcommand{\nb}[3]{{\colorbox{#2}{\bfseries\sffamily\scriptsize\textcolor{white}{#1}}}{\textcolor{#2}{\sf\small\textit{#3}}}}

\newcommand{\ab}[1]{\nb{Anas}{blue}{#1}}
\newcommand{\nl}[1]{\nb{Nicolas}{red}{#1}}

\newcommand\margincomment[4]{{\color{#3}#1}\marginpar{\vspace{-0\baselineskip}\color{#3}\raggedright\scriptsize [#4]:#2}}

\newcommand{\atmargin}[2]{\margincomment{#1}{#2}{orange}{AT}}

\newcommand{\colorppo}{orange}
\newcommand{\colortrpo}{red}
\newcommand{\colorwpo}{grgreen}
\newcommand{\colorwtrpo}{blue}
\newcommand{\colorac}{magenta}
\newcommand{\colorbgpg}{brown}
\newcommand{\colorwnpg}{yellow}

\title{
Trust Region Policy Optimization \\ with Optimal Transport Discrepancies: \\
Duality and Algorithm for Continuous Actions
}

\author{%
    Antonio Terpin\thanks{Equal contribution.} \\
    Automatic Control Laboratory \\ 
    ETH Zürich \\
    \texttt{aterpin@ethz.ch}
    \And
    Nicolas Lanzetti\textsuperscript{*}\\
    Automatic Control Laboratory \\ 
    ETH Zürich \\
    \texttt{lnicolas@ethz.ch}
    \And
    Batuhan Yardim \\
    Dept. of Computer Science \\
    ETH Zürich \\
    \hspace{-0.5cm}\texttt{alibatuhan.yardim@ethz.ch}
    \AND
    Florian D\"{o}rfler \\
    Automatic Control Laboratory \\
    ETH Zürich \\
    \texttt{dorfler@ethz.ch}
    \And
    Giorgia Ramponi \\
    Dept. of Computer Science\\
    ETH AI Center \\
    \texttt{giorgia.ramponi@ai.ethz.ch}
}

\begin{document}

\maketitle

\begin{abstract}
Policy Optimization (PO) algorithms have been proven particularly suited to handle the high-dimensionality of real-world continuous control tasks. In this context, Trust Region Policy Optimization methods represent a popular approach to stabilize the policy updates. These usually rely on the Kullback-Leibler (KL) divergence to limit the change in the policy. The Wasserstein distance represents a natural alternative, in place of the KL divergence, to define trust regions or to regularize the objective function. However, state-of-the-art works either resort to its approximations or do not provide an algorithm for continuous state-action spaces, reducing the applicability of the method.
In this paper, we explore optimal transport discrepancies (which include the Wasserstein distance) to define trust regions, and we propose a novel algorithm -- Optimal Transport Trust Region Policy Optimization (OT-TRPO) -- for continuous state-action spaces. We circumvent the infinite-dimensional optimization problem for PO by providing a one-dimensional dual reformulation for which strong duality holds.
We then analytically derive the optimal policy update given the solution of the dual problem. This way, we bypass the computation of optimal transport costs and of optimal transport maps, which we implicitly characterize by solving the dual formulation.
Finally, we provide an experimental evaluation of our approach across various control tasks. Our results show that optimal transport discrepancies can offer an advantage over state-of-the-art approaches.
\end{abstract}

\section{Introduction}
\gls*{acr:rl} has achieved outstanding results in numerous fields, from resource management~\citep{mao2016resourcemanagement}, recommendation systems~\citep{Guanjie2018recommendation}, and optimization of chemical reactions~\citep{zhou2017chemical}, to video-games~\citep{mnih2013atari, yue2019videogames, Kaiser2020atari} and board games~\citep{xenou2019boardgames}, without sparing the world's champion of GO \citep{silver2016go}. Many of these successful applications rely on \gls*{acr:po} algorithms, a family of \gls*{acr:rl} methods that are particularly suited to handle the high-dimensionality of real-world control tasks.
\gls*{acr:po} algorithms approach the \gls*{acr:rl} setting as an optimization problem in the policy space. In this context, the main challenge is to provide policy improvement guarantees. One remarkable option in this direction is represented by~\gls*{acr:trpo}~\cite{Schulman2015}, which constrains the optimization problem to policies that are ``close'' to the current one, whereby the \gls*{acr:kl} divergence is used as a similarity measure. Nevertheless, ``closeness'' in the policy space can also be quantified via other functions.
Recent work~\cite{Pacchiano2020, Song2022, Moskovitz2020} proposed to replace the \gls*{acr:kl} divergence with the Wasserstein distance, a particular instance of optimal transport discrepancy (or cost). 
Besides being very natural and expressive, optimal transport discrepancies enjoy powerful topological, differential, geometrical, computational, and statistical features and guarantees~\cite{Villani2008,Aolaritei2022,lanzetti2022first}.
In particular, (i) optimal transport discrepancies allow us to compare probability measures (and thus policies) not sharing the same support (for which the \gls*{acr:kl} divergence is infinity); and (ii) they encapsulate the geometry encoded by the transport cost in the action space: the discrepancy between two actions coincides with the discrepancy between the corresponding deterministic policies (whereas the \gls*{acr:kl} divergence is again infinity). These reasons make optimal transport discrepancies particularly attractive for~\gls*{acr:rl}. However, the mere evaluation of optimal transport discrepancies entails solving a transportation problem (e.g., see~\cite{Peyre2019,taskesen2021semi}), which poses significant computational challenges for its deployment.
Most of the previous work on the topic~\cite{Pacchiano2020, Moskovitz2020} overcomes the computational burden via approximation, effectively changing the original problem.
Conversely, \cite{Song2022} proposes two algorithms to solve the \gls*{acr:po} problem exactly, studying the trust regions described via the Wasserstein distance and the Sinkhorn divergence. However, their analysis is limited to discrete (and finite) settings.
In our work, we consider optimal transport discrepancies to construct the trust region in settings where actions and states take value in general compact Polish spaces. This allows tackling many applications of interest involving continuous domains such as physical control tasks. We derive and leverage a dual reformulation of the \gls*{acr:po} problem to ensure an optimal policy update within the trust region, without any additional need for line searches (conversely to~\cite{Schulman2015}). We circumvent the computation of optimal transport discrepancies via an analytical expression of the transport maps, which are characterized thanks to the dual reformulation.
Notably, our analysis enables a practical and efficient algorithm that encompasses both discrete and continuous settings.

\textbf{Contributions.} Our contributions are summarized as follows:
\begin{enumerate}[leftmargin = *]
\item We derive the dual of the optimal transport trust region policy optimization problem and we show that strong duality holds for general compact metric state-action spaces. We further characterize the optimal policy update given the solution to the dual problem. We show that policy updates can result in monotonic improvement of the performance function.

\item We propose a novel \gls*{acr:po} algorithm for continuous spaces, \gls*{acr:ottrpo}. Herein, we leverage the derived duality theory to provide policy updates that satisfy the optimal transport discrepancy constraint while circumventing its computation.

\item We conduct experiments in several \gls*{acr:rl} benchmarks in both discrete and continuous state-action spaces, comparing our method to state-of-the-art approaches. Our results show the effectiveness of our approach for~\gls*{acr:po} and the benefits of using optimal transport discrepancies. 
\end{enumerate}

\section{Related Works}
Optimal transport, and in particular the Wasserstein distance, has found various applications in \gls*{acr:rl} and in particular in ~\gls*{acr:po} algorithms. In this section, we discuss the most relevant for our work; a broader overview is postponed to~\cref{app:related-work}.
In~\cite{Pacchiano2020}, the authors propose \gls*{acr:bgpg}, whereby they replace the \gls*{acr:kl} divergence trust region from \gls*{acr:trpo}~\cite{Schulman2015} by a Wasserstein distance penalty in a behavioral space.
Although our alternating procedure in~\cref{sec:practical-ottrpo} may resemble the approach of~\cite{Pacchiano2020} in spirit, our approach is fundamentally different; cfr. \cite[Algorithms 1 and 3]{Pacchiano2020} with \cref{algo:policy-update}.  
In \cite{Moskovitz2020}, the authors further suggested incorporating additional information about the local behavior of policies encapsulated in the so-called Wasserstein Information Matrix, in the attempt to speed up the~\gls*{acr:po} using a \gls*{acr:wnpg}. However, these approaches are relatively slow, compared to the traditional \gls*{acr:ppo} and \gls*{acr:trpo}. Conversely to our work, they do not build on the idea of trust regions: we instead guarantee that the policy update is ``close'' to the previous one, where the ``closeness'' is defined via an optimal transport discrepancy (e.g., the Wasserstein distance).
Accordingly, the closest related work to ours is the recent paper~\cite{Song2022} which studied~\gls*{acr:wpo} for discrete action spaces. In contrast to~\cite{Song2022}, the present work addresses the general setting of compact Polish spaces encompassing the cases of continuous and discrete state-action spaces as particular cases. While we also adopt a duality approach, our level of generality induces many challenges compared to the discrete action space setting (see Remark~\ref{rem:comparison-song-1}), and it is compatible even with non-direct policy parametrizations.
Finally, our work is closely connected with Wasserstein \gls*{acr:dro}~\cite{Gao2016,Mohajerin2018,blanchet-murthy2019,rahimian2019distributionally,kuhn-et-al2019survey}. Albeit our duality results are inspired from this literature, \gls*{acr:dro} is concerned with quantifying the worst-case risk of a cost functional over an ambiguity set of probability measures, which is a fundamentally different setting; see~\cref{rem:comparison-dro}. Conversely to all the previous work, we show how to perform exact optimal transport-based \gls*{acr:trpo} in continuous settings: we exploit optimal transport theory to circumvent the computational burden of evaluating optimal transport discrepancies while still performing exact policy updates within the trust regions. Accordingly, to the best of our knowledge, our practical algorithm is completely novel.

\section{Preliminaries}
\label{sec:preliminaries}

We briefly introduce useful background and notation for the remainder of the paper. 

\noindent\textbf{Notation.} 
For every Polish space~$\mathcal{X}$ (i.e., completely metrizable separable topological space), the set of Borel probability measures on~$\mathcal{X}$ is denoted by $\probabilityspace{\mathcal{X}}$. The Dirac measure at some point~$x \in \mathcal{X}$ is denoted by~$\delta_x$.  Given two Polish spaces $\mathcal{X},\mathcal{Y}$, a Borel probability measures $\mu \in \probabilityspace{\mathcal{X}}$, and a Borel map $T: \mathcal{X} \to \mathcal{Y}$, the pushforward measure of $\mu$, denoted by $\pushforward{T}{\mu}$, is defined by~$(\pushforward{T}{\mu})(A)\coloneqq\mu(T^{-1}(A))$ for all $A \in \borel{\mathcal{Y}}$, where $\borel{\mathcal{Y}}$ is the collection of Borel subsets of $\mathcal{Y}$.
The set of probability measures on a finite set~$\mathcal{X}$ coincides with the probability simplex and will also be denoted by~$\probabilityspace{\mathcal{X}}$. For a given function~$f: \mathcal{X} \to \bR$, the notation $\|f\|_{\infty}$ refers to~$\sup_{x \in \mathcal{X}} |f(x)|$.

\noindent\textbf{Markov Decision Process.}
We consider an infinite-horizon discounted~\gls*{acr:mdp}~\cite{puterman2005} $\markov = (\statespace, \actionspace, \transitionprobability, \reward, \rho, \gamma)$, where $\statespace$ is the state space, $\actionspace$ is the action space, $\transitionprobability: \statespace\times\actionspace\times\statespace \to \nonnegativeReals$ is the state transition probability kernel, $\reward:\statespace\times\actionspace\to\reals$ is the reward function, $\rho$ is the initial state probability distribution, and $\gamma \in [0,1)$ is the discount factor. 
A randomized stationary Markovian policy, which we will simply call a policy in the rest of the paper, is a mapping $\pi: \statespace \to \mathcal{P}(\actionspace)$ specifying for each~$s \in \statespace$ a probability measure over the set of actions $\actionspace$ by~$\pi(\cdot|s) \in \probabilityspace{\actionspace}$.
The set of all policies is denoted by $\policies$.
Each policy $\pi \in \Pi$ induces a discrete-time Markov reward process $\{(s_t, r(s_t,a_t))\}_{t \in \mathbb{N}}$, where $s_t \in \statespace$ represents the state of the system at time~$t$ and~$r(s_t,a_t)$ corresponds to the reward received when executing action $a_t \in \actionspace$ in state~$s_t$.
We denote by~$\bP_{\rho,\pi}$ the probability distribution of the Markov chain~$(s_t,a_t)$ issued from the \gls*{acr:mdp} controlled by the policy~$\pi$ with initial state distribution~$\rho$. The associated expectation is denoted by~$\bE_{\rho,\pi}$ and the notation~$\bE_{\pi}$ is used whenever there is no dependence on~$\rho$. The state-value function $V^\pi : \statespace \to \bR$ and the action-value function $Q^\pi : \statespace \times \actionspace \to \bR$ are defined for all $s\in \statespace, a \in \actionspace$ by
$
V^\pi(s) \eqdef \bE_\pi [ \sum_{t=0}^\infty \gamma^t r(s_t,a_t)| s_0 = s]
$
and
$
Q^\pi(s,a) \eqdef \bE_\pi [ \sum_{t=0}^\infty \gamma^t r(s_t,a_t) | s_0 = s, a_0= a ].
$
We also define the advantage function $\advantage : \statespace \times \actionspace \to \bR$ by $\advantage(s,a) \eqdef Q^\pi(s,a) - V^\pi(s)$.
Given an initial state probability distribution $\rho$, our goal is to find a policy $\pi$
maximizing the expected long-term return
\begin{equation*}
J(\pi) \coloneqq \bE_{\rho,\pi} \left[\sum_{t = 0}^\infty \gamma^t r(s_t, a_t)\right],
\end{equation*}
which is well-defined when, e.g., the reward function is bounded.
To solve this~\gls*{acr:po} problem, we only have access to the observed state, action, and reward ~$s_t, a_t, r_t$ at each time step~$t$, whereas the state transition kernel $\transitionprobability$ is unknown.
When the state and action spaces~($\statespace$ and~$\actionspace$) are finite, an optimal policy~$\pi^*$ is guaranteed to exist. When $\statespace$ and~$\actionspace$ are continuous, a (measurable) optimal policy is also guaranteed to exist (see \cite[Theorem~6.11.11, p. 262]{puterman2005}) under appropriate assumptions on the state and action spaces, the reward function and the transition kernel; we will explicit these later on. 
In this paper, we focus on the continuous state-action space setting and comment on the discrete (non necessarily finite) setting as a special case.

\noindent\textbf{Optimal transport.} Consider a Polish space~$\mathcal{X}$ and a continuous non-negative function $c:\mathcal X\times \mathcal X\to\mathbb R_{\geq 0}$, referred to as \emph{transport cost}.
Let~$\mu, \nu \in \probabilityspace{\mathcal{X}}$ and define the set of joint probability measures on $\mathcal{X} \times \mathcal{X}$ with marginals~$\mu$ and~$\nu$:
\begin{equation*}
\Gamma(\mu, \nu) \eqdef \{ \gamma \in \probabilityspace{\mathcal{X} \times \mathcal{X}}: \gamma(A \times \mathcal{X}) = \mu(A), \gamma(\mathcal{X} \times B) = \nu(B)\: \forall\, A,B\in\borel{\mathcal{X}}\}.
\end{equation*}
We define the \emph{optimal transport discrepancy} on~$\probabilityspace{\mathcal{X}}$ for every probability measures~$\mu$ and~$\nu$ by
\begin{equation}\label{eq:otdiscrepancy}
\otdiscrepancy{\mu}{\nu}
\coloneqq
\min_{\gamma \in \Gamma(\mu,\nu)}\int_{\mathcal{X}\times \mathcal{X}}c(x,x')\, \d\gamma(x,x').
\end{equation}
Notice that this definition is valid for both discrete and continuous measures. When $c=d^p$, where~$d$ is a distance on~$\mathcal X$ and $p\geq 1$, then $\otdiscrepancy{\mu}{\nu}^{1/p}$ reduces to the celebrated type-$p$ Wasserstein distance~\cite{Villani2008,Ambrosio2008}. 
In our~\gls*{acr:po} context, we will use this discrepancy to compare two probability measures~$\pi(\cdot|s) \in \probabilityspace{\actionspace}$ and~$\tilde{\pi}(\cdot|s) \in \probabilityspace{\actionspace}$ for every $s \in \statespace$, where~$\pi, \tilde{\pi} \in \Pi$ are two policies. 

\section{Optimal Transport for Trust Region Policy Optimization}
\label{sec:ot-trpo}
In this section, we study the \gls*{acr:trpo} algorithm with a trust region defined using an optimal transport discrepancy as a measure of closeness between policies. We prove that the arising optimization problem admits an amenable dual reformulation. Importantly, we show that, given the dual optimal solution, the primal solution has an analytical expression, which can lead to monotonic improvements of the performance index.

\subsection{Policy iteration algorithm with optimal transport-based trust regions}
By the policy difference lemma~\cite[Lemma 6.1]{Kakade2002b}, the difference between the expected returns of two policies $\pi, \tilde{\pi} \in \Pi$ reads
\begin{equation}
\label{eq:perf-diff-lemma}
    J(\tilde{\pi})
    =
    J(\pi) + \int_{\statespace} \int_{\actionspace}\advantage(s,a) \d\tilde{\pi}(a|s) \d\rho_{\tilde{\pi}}(s),
\end{equation}
where~$\rho_{\tilde{\pi}}$ is the discounted state-occupancy measure~\cite{Kakade2002b}.
The complex dependency of the discounted visitation frequency $\rho_{\tilde{\pi}}$ on the policy $\tilde{\pi}$ hampers the direct optimization of~\eqref{eq:perf-diff-lemma}; see~\cite[Sec.~2]{Schulman2015}. Following previous work, we consider instead a local approximation of the expected return~$J$, defined by
\begin{equation}
\label{eq:surrogate}
    L_{\pi}(\tilde{\pi})
    \eqdef
    J(\pi) + \int_{\statespace} \int_{\actionspace} \advantage (s,a) \d\tilde{\pi}(a|s)\d\visitationfreq(s).
\end{equation}
Observe that this approximation uses the discounted state-occupancy measure~$\rho_{\pi}$ (which can be estimated) instead of~$\rho_{\tilde{\pi}}$ (see~\eqref{eq:perf-diff-lemma}). In other words, the influence of a policy change on the discounted state-occupancy measure is neglected.
Moreover, this surrogate function coincides with the expected return~$J$ when $\tilde{\pi} = \pi$. Then,~\eqref{eq:surrogate} motivates a policy update rule maximizing at each time step the approximation $L_{\pi}(\tilde{\pi})$ over $\tilde{\pi}$, where~$\pi$ is the current policy that we want to improve upon (see also~\cite[Section 2, Eq.~(1)]{Song2022}).
To ensure stability of the update, we conservatively update the policy using a discrepancy constraint between the current and the new one.
Unlike~\gls*{acr:trpo}, we do not use the~\gls*{acr:kl} divergence to define the trust region, but instead an optimal transport discrepancy. Then, at each time step, our method solves
\begin{equation}
\begin{aligned}
\label{pb:trcp}    
    \sup_{\tilde{\pi} \in \Pi} &\int_{\statespace} \int_{\actionspace}  \advantage(s,a)
    \d\tilde{\pi}(a|s)\d\visitationfreq(s),
    \\
    \text{s.t. } &\tilde\pi\in\mathcal{T}_{\varepsilon}(\pi) \eqdef \left\{ \tilde{\pi} \in \Pi:
    \int_{\statespace} \otdiscrepancy{\pi(\cdot|s)}{\tilde\pi(\cdot|s)}\d\visitationfreq(s) \leq \varepsilon \right\},\\
\end{aligned}\tag{P}
\end{equation}
where $\varepsilon >0$ is a parameter defining the radius of the trust region~$\mathcal{T}_{\varepsilon}(\pi)$.
Similarly to \cite[Eq. (12)]{Schulman2015} and \cite[Problem~(4)]{Song2022}, we consider the \emph{average} optimal transport discrepancy over the state space as optimization constraint.
Accordingly, the \gls*{acr:ottrpo} policy optimization results from iteratively solving Problem~\eqref{pb:trcp}.

\subsection{Dual of the trust-region constrained problem~(\ref{pb:trcp})}
Problem~(\ref{pb:trcp}) is intractable for two main reasons.
First, as soon as the state or action space is continuous, it is an infinite-dimensional optimization problem.
Second, the mere evaluation of the trust-region constraint (e.g., for line search as in~\gls*{acr:trpo}~\cite{Schulman2015}) needs (possibly) infinitely many computations of the optimal transport discrepancy, which is itself already challenging to estimate.
However, inspired by prior works on Wasserstein \gls*{acr:dro}~\cite{Gao2016,Mohajerin2018,zhao-guan2018,blanchet-murthy2019}, we show that problem~\eqref{pb:trcp} admits a tractable one-dimensional convex dual reformulation. This duality theorem is the cornerstone of the design of our algorithm. Before stating the result, we make the following assumptions.
\begin{assumption}
\label{hyp:state-action-space}
The state space~$\statespace$ is a compact subset of an Euclidean space, the action space~$\actionspace$ is a compact subset of a Polish space, the reward function~$r$ is a continuous function and for every  continuous function~$w$ on~$\statespace$, $\int_{\statespace} w(u) \d\transitionprobability(u|s,a)$ is continuous in both~$s$ and~$a$.
\end{assumption}
Under this assumption, there exists an optimal measurable (stationary) policy to the~\gls*{acr:po} problem formulated in~\cref{sec:preliminaries}. We refer the reader to~\cite[Theorem~6.11.11, p. 262]{puterman2005} for a statement of this result and milder assumptions. In particular, our duality result continues to hold if $\statespace$ is a compact Polish space (i.e., not necessarily Euclidean). 
\begin{assumption}
\label{hyp:continuous-advantage}
For every policy~$\pi \in \Pi$, the advantage function $\advantage: \statespace \times \actionspace \to \reals$ is continuous. Moreover, the transport cost $c:\actionspace\times\actionspace\to\nonnegativeReals$ is continuous and satisfies $c(a,a)=0$ for all $a\in\actionspace$.
\end{assumption}
In the next theorem, we show that under these assumptions Problem~\eqref{pb:trcp} admits a dual reformulation for which strong duality holds.
\begin{theorem}[Dual formulation]
\label{thm:dual}
For every~$\varepsilon >0$ and for every policy~$\pi \in \Pi$, under~\cref{hyp:state-action-space,hyp:continuous-advantage} the following strong duality result holds: 
\begin{align}
    \max_{\tilde{\pi} \in \Pi} &\left\{\int_{\statespace} \int_{\actionspace}\advantage(s,a) \, \d\tilde{\pi}(a|s)\d\visitationfreq(s):
    \int_{\statespace} \otdiscrepancy{\pi(\cdot|s)}{\tilde\pi(\cdot|s)} \d\visitationfreq(s) \leq \varepsilon \right\}\, \label{pb:primal}\tag{P}\\
     &=
     \min_{\lambda \geq 0}\left\{ \lambda\varepsilon + \int_{\statespace} \int_{\actionspace} \max_{a' \in \actionspace}\{\advantage(s,a') - \lambda c(a,a')\}\d\pi(a|s) \d\visitationfreq(s) \right\}. \label{pb:dual}\tag{D}
\end{align}
Moreover, the primal and dual problems~\eqref{pb:trcp} and~\eqref{pb:dual} admit a maximizer and a minimizer, respectively. 
\end{theorem}

Remarkably, Problem~\eqref{pb:dual} is one-dimensional and convex, and it only involves the current policy $\pi$, advantage function $\advantage$, and visitation frequency $\visitationfreq$. The proof of~\cref{thm:dual} is constructive. In particular, we derive a closed-form solution of problem~(\ref{pb:trcp}) as a function of the optimal Lagrange multiplier $\lambda^*$ solving problem~(\ref{pb:dual}) and the policy~$\pi$ defining the problem. Even if the closed-form policy is part of~\cref{thm:dual} and its proof, we present it separately for clarity and for later reference. To do so, we introduce some additional notation, which is instrumental to derive a practical algorithm (similarly to~\cite{Song2022}). 

Under~\cref{hyp:state-action-space,hyp:continuous-advantage}, define for every~$\lambda \geq 0$ the $\lambda$-regularized advantage~$\regularizer_{\lambda}: \statespace \times \actionspace \to \bR$ and its associated set of maximizers for every $s \in \statespace, a \in \actionspace$ as follows: 
\begin{equation}
\begin{aligned}
\label{eq:def-Phi-D}
    \regularizer_{\lambda}(s,a) &\coloneqq
    \max_{a' \in \actionspace} \{\advantage(s,a') - \lambda c(a,a')\},
    \\
    \mins{\lambda}{s}{a} &\coloneqq
    \argmax_{a' \in \actionspace} \{\advantage(s,a') - \lambda c(a,a')\}. 
\end{aligned}
\end{equation}

\begin{corollary}[Optimal policy]
\label{cor:optimal-policy}
Under the setting and assumptions of~\cref{thm:dual}, for any policy~$\pi \in \Pi$, let~$\lambda^*\geq 0$ be the minimizer of the dual problem~\eqref{pb:dual}. Then, the following statements hold:
\begin{enumerate}[leftmargin=*,noitemsep,topsep=0pt]
    \item For every $\lambda \geq 0$, there exist two measurable selection maps~$\optmaplowername{\lambda}: \statespace \times \actionspace \to \actionspace$ and~$\optmapuppername{\lambda}:\statespace \times \actionspace \to \actionspace$ such that for every~$s \in \statespace, a \in \actionspace$
    \begin{equation}
    \optmaplower{\lambda}{s}{a} \in \argmin_{a' \in \mins{\lambda}{s}{a}}c(a,a'),
    \quad 
    \optmapupper{\lambda}{s}{a} \in \argmax_{a' \in \mins{\lambda}{s}{a}}c(a,a').
    \end{equation}
    
    \item If $\lambda^\ast > 0$, there exists~$t^* \in [0,1]$ such that
    \begin{equation}\label{eq:t-star}
    t^\ast\int_\statespace\int_\actionspace c(a,\optmaplower{\lambda^*}{s}{a}) \d\pi(a|s) \d\rho_{\pi}(s)
    +
    (1-t^\ast) \int_\statespace\int_\actionspace c(a,\optmapupper{\lambda^*}{s}{a}) \d\pi(a|s) \d\rho_{\pi}(s)
    = \varepsilon.
    \end{equation}

    \item There exists an optimal feasible policy~$\tilde{\pi}$ for problem~\eqref{pb:primal} defined for every~$s \in \statespace$ by
    \begin{equation}
    \label{eq:optimal-policy-update}
        \tilde{\pi}(\cdot|s)
        \eqdef
        t^* \optmaplower{\lambda^*}{s}{\cdot}_\# \pi(\cdot|s)
        + (1-t^*) \optmapupper{\lambda^*}{s}{\cdot}_\# \pi(\cdot|s),
    \end{equation}
    where $t^\ast$ results from~\eqref{eq:t-star} if $\lambda^\ast>0$ and $t^\ast=0$ if $\lambda^\ast = 0$.
\end{enumerate}
\end{corollary}
Intuitively,~\cref{cor:optimal-policy} suggests that the optimal policy results from displacing the probability mass $\pi(a|s)$ to the maximizers of the $\lambda^\ast$-regularized advantage $\regularizer_{\lambda^\ast}(s,a)$, where $\lambda^\ast\geq 0$ is the optimal dual solution. Since maximizers are generally not unique (i.e., $\mins{\lambda^\ast}{s}{a}$ is not a singleton), one needs to balance between the \emph{closest} (i.e., $\optmaplower{\lambda^\ast}{s}{a}$) and the \emph{furthest apart} (i.e., $\overline{T}_{\lambda^\ast}(s,a)$) to satisfy the trust region constraint.
In the special case $\lambda^\ast=0$, the trust region constraint is either not active (i.e., the optimal policy lies within the trust region) or it does not affect the optimal trust region constraint (i.e., the optimal policy would lie at the boundary of the trust region even if the constraint is removed). In this case, it suffices to displace all probability mass to the closest maximizer of the advantage functions (i.e., $\optmaplower{0}{s}{a}\in\mins{0}{s}{a}$). 
The complete proof of the results of this section is deferred to~\cref{appendix:strongduality}.

\begin{remark}
\label{rem:comparison-dro}
Similar results were previously established in the literature in the context of \gls*{acr:dro} (e.g., see~\cite[Theorem 1]{Gao2016} and~\cite[Theorem 1]{blanchet-murthy2019}). While these results closely inspire our proof, there is a major difference: in~\gls*{acr:dro}, one seeks to evaluate the worst-case cost over an ambiguous set of probability distributions, expressed in terms of the optimal transport discrepancy. As such, the \emph{average} optimal transport discrepancy in the trust region constraint is replaced by a single optimal transport discrepancy. Thus, one does not need to ensure the regularity of the problem with respect to the state (e.g., measurability of $\optmaplowername{\lambda}$ w.r.t. $s$). This is reflected in our assumption of \emph{joint} continuity of the advantage in state and action and the state space being compact. To readily deploy existing results in \gls*{acr:dro}, one needs (i) to consider a single state only (i.e., $\statespace=\{s\}$) or (ii) to define a trust region \emph{for each} state (at the price of infinitely many constraints). 
To transform~\eqref{pb:primal} into a~\gls*{acr:dro}, one might alternatively identify a policy as a probability measure over $\prod_{s\in\statespace}\actionspace$, and hope to deploy standard duality arguments in~\gls*{acr:dro}.
However, the uncountable product of Polish spaces is not a Polish space, which makes all results in~\gls*{acr:dro}, and more generally in optimal transport~\cite{Villani2008}, inapplicable.
\end{remark}

\begin{remark}
\label{rem:comparison-song-1}
A similar result in the discrete case was presented in~\cite{Song2022}. We highlight four major differences.
First, in the discrete setting, problem \eqref{pb:primal} is a finite-dimensional linear optimization problem, for which strong duality holds. Thus, linear programming arguments can be used to derive the dual reformulation. In the continuous setting, the same proof strategy would require to mobilize the abstract machinery of infinite-dimensional linear programming~\cite{Leuenberger1997}.
Second, in the discrete setting, continuity and measurability of all functions are ``for free''. On the contrary, the continuous case imposes a careful analysis of these issues. 
Third, the optimal policy update~\cite{Song2022} implicitly assumes that the set~$\mathcal{D}_{\lambda}(s,a)$ (see~\eqref{eq:def-Phi-D}) is a singleton, which is rarely satisfied in practice. 
Fourth, as a byproduct of our proof, we show that~\eqref{pb:dual} is a (one-dimensional) \emph{convex} optimization problem, which can be solved efficiently via off-the-shelf solvers. This way, we do not need to resort to approximation techniques~\cite[Section 6.1]{Song2022} for the optimal dual multiplier. 
\end{remark}

\noindent\textbf{Discrete state-action spaces.}
In the remainder of this section, we specialize our results to discrete (finite) state-action spaces (which trivially satisfy~\cref{hyp:state-action-space,hyp:continuous-advantage}).
Without loss of generality, we represent the state and action spaces by~$\statespace = \{s_1, \ldots, s_M\}$ and~$\actionspace = \{a_1, \ldots, a_N\}$ where~$M$ and~$N$ are two positive integers, and we describe any policy~$\pi \in \Pi$ and its corresponding state-occupancy measure~$\rho_\pi$ as discrete measures:
\begin{equation}
\label{eq:discrete-measures}
\pi(\cdot|s_i)=\sum_{j=1}^N \pi_{i,j} \diracdelta{a_j}
\qquad
\forall i \in \{1, \ldots, M \},
\qquad
\rho_{\pi} = \sum_{i=1}^M \rho_i \diracdelta{s_i},
\end{equation}
where $\rho_i, \pi_{i,j} \geq 0$ for every $i \in \{1, \ldots, M\},\, j \in \{1, \ldots, N\}$\,, $\sum_{i=1}^M \rho_i = 1$ and~$\sum_{j=1}^N \pi_{i,j} = 1$ for every~$i \in \{1, \ldots, M\}$.\footnote{This representation is also valid beyond the finite state-action space setting when the policies and the state-occupancy measures are empirical distributions with finitely many samples.} The analogous results to~\cref{thm:dual} and~\cref{cor:optimal-policy}  are as follows.

\begin{corollary}[Dual formulation - discrete setting]
\label{cor:dual-discrete}
Let $\varepsilon >0$. For every policy~$\pi \in \Pi$, the following strong duality result holds: 
\begin{align}
    \max_{\substack{t \in [0,1],\underline{b}_{i,j},\overline{b}_{i,j} \in \actionspace,\\
    i \in  \{1, \ldots M \},\, j \in \{1, \ldots, N \}}}
    &
    \begin{aligned}[t]
    &\left\{\sum_{i=1}^M \rho_i \sum_{j=1}^N \pi_{i,j}
    \left(t \advantage(s_i,\underline{b}_{i,j}) + (1-t) \advantage(s_i,\overline{b}_{i,j})\right):\right.
    \\
    &\:\:\;\left.\sum_{i=1}^M \rho_i \sum_{j=1}^N \pi_{i,j}
    \left(t c(a_j,\underline{b}_{i,j}) + (1-t) c(a_j,\overline{b}_{i,j})\right) \leq \varepsilon\right\}
    \end{aligned}
    \label{pb:primal-discrete}\tag{discrete-P}
    \\
    &=
    \min_{\lambda \geq 0} \left\{ \lambda \varepsilon + \sum_{i=1}^M \rho_i \sum_{j=1}^N \pi_{i,j} \regularizer_{\lambda}(s_i, a_j) \right\}. \label{pb:dual-discrete}\tag{discrete-D}
\end{align}
In particular, let~$\lambda^*\geq 0$ be a solution to~\eqref{pb:dual-discrete}, and given for every~$i \in \{1, \ldots, M\},\, j \in \{1, \ldots, N\}$, select any
\begin{equation}
\label{eq:bij*}
    \underline{b}_{i,j}^* \in  \argmin_{a' \in \mins{\lambda^*}{s_i}{a_j}}c(a_j,a'), \qquad 
    \overline{b}_{i,j}^* \in \argmax_{a' \in \mins{\lambda^*}{s_i}{a_j}}c(a_j,a'),
\end{equation}
and let
\begin{equation*}
    \underline{c}
    \coloneqq
    \sum_{i = 1}^M\rho_i\sum_{j = 1}^N \pi_{i,j} c(a_j, \underline{b}_{i,j}^\ast),
    \qquad 
    \overline{c}
    \coloneqq 
    \sum_{i = 1}^M\rho_i\sum_{j = 1}^N \pi_{i,j} c(a_j, \overline{b}_{i,j}^\ast).
\end{equation*}
Then, an optimal policy~$\tilde{\pi}$ is given by
\begin{equation}
\label{eq:opt-policy-discrete}
\tilde{\pi}(\cdot|s_i) = \sum_{j=1}^N \pi_{i,j} \left(t^* \diracdelta{\underline{b}_{i,j}^*} + (1-t^*) \diracdelta{\overline{b}_{i,j}^*}\right),
\qquad
\forall i\in\{1,\ldots,M\},
\end{equation}
with $t^\ast = (\overline{c} - \varepsilon) / (\overline{c}-\underline{c})\in[0,1]$ (and $t^\ast\in[0,1]$ if $\underline{c} = \overline{c}=\varepsilon$) if $\lambda^\ast>0$ and $t^\ast=0$ if $\lambda^\ast=0$.
\end{corollary}

The proof of this result stems from substituting the discrete measures as defined in~\eqref{eq:discrete-measures} in problems~\eqref{pb:primal} and~\eqref{pb:dual}, and observing that the images of the mappings $\optmaplowername{\lambda^*}$ and $\optmapuppername{\lambda^*}$ have finite support in the current setting.
Notably, \cref{cor:dual-discrete} directly provides an implementable algorithm for the policy update, circumventing the difficulty of the mixed-integer optimization problem~\eqref{pb:primal-discrete}: solving the one-dimensional convex program~\eqref{pb:dual-discrete} provides the optimal Lagrange multiplier associated to the trust region constraint of the primal problem which can be directly used to compute the actions~$\underline{b}_{i,j}^*,\, \overline{b}_{i,j}^*$ via~\eqref{eq:bij*}, and thus the policy~$\tilde{\pi}$ via~\eqref{eq:opt-policy-discrete}.

\begin{remark}\label{rem:comparison-song-2}
The policy update suggested by~\eqref{eq:opt-policy-discrete} differs from the one in~\cite{Song2022} (see \cite[Theorem~1,~(5)]{Song2022} where $f_s^*(i,j) \in \{0,1\}$ with their notation). Indeed, our policy update relies on ``splitting the probability mass'': the probability mass $\pi(\cdot|s_i)$ is displaced to~$\underline{b}_{i,j}^*$ and~$\overline{b}_{i,j}^*$ with weights~$t^\ast$ and~$1-t^\ast$, respectively. 
This result is consistent with the Wasserstein \gls*{acr:dro} literature (e.g., see~\cite[Remark 2]{blanchet-murthy2019}).
The result provided in~\cite{Song2022} corresponds to the particular case where~$\underline{b}_{i,j}^*=\overline{b}_{i,j}^*$ which amounts to supposing that the set~$\mathcal D_\lambda(s_i,a_j)$ defined in~\eqref{eq:def-Phi-D} is a singleton. 
We provide further comments and examples in~\cref{appendix:mass-splitting} to illustrate the importance of this ``mass splitting''. 
\end{remark}

\subsection{Policy improvement}
In the next result, we show that our policy update leads to a monotonic improvement of the performance function~$J$ up to the advantage function estimation error. 
\begin{proposition}[Performance improvement]
\label{prop:exact-policy-improv}
Let $\pi \in \Pi$. Consider solutions $\tilde{\pi}^* \in \Pi$ and $\lambda^* \geq 0$ of problems~\eqref{pb:trcp} and~\eqref{pb:dual}, respectively. If the true advantage function~$\advantage$ is approximated by some estimated advantage function~$\advantagehat$ such that $\|\advantage - \advantagehat\|_{\infty} < \infty$, then the following bound holds: 
\begin{equation}
    J(\tilde{\pi}^*)
    \geq
    J(\pi) + \frac{\lambda^*}{1-\gamma} \int_{\statespace} \otdiscrepancy{\pi(\cdot|s)}{\tilde\pi^\ast(\cdot|s)}\d\rho_{\tilde{\pi}^*}(s) - \frac{2 \|\advantage - \advantagehat\|_{\infty}}{1 - \gamma}.
\end{equation}
\end{proposition}
\cref{prop:exact-policy-improv} indicates that optimal transport-based trust region policy optimization leads to monotonic improvement of the performance function when we have access to the true advantage function.
The proof, postponed to~\cref{appendix:proof-policy-improvement}, of this result builds on the performance difference lemma (see~\eqref{eq:perf-diff-lemma}) and uses the closed-form expression of the optimal policy solving problem~(\ref{pb:trcp}) as constructed in the proof of~\cref{thm:dual} (see~\cref{cor:optimal-policy}). The analog of this result for a finite action space was proved in \cite[Theorem~2, p.~5]{Song2022}. To the best of our knowledge, this result is novel for the continuous state-action space setting.

\section{Practical Optimal Transport Trust Region Policy Optimization Algorithm}
\label{sec:practical-ottrpo}

In this section, we use the duality results on~\cref{sec:ot-trpo} to derive a practical algorithm for~\gls*{acr:ottrpo}.
Herein, we restrict the policy search set $\Pi$ to the set of policies~$\pi_{\theta}$ parametrized by a vector~$\theta \in \mathbb{R}^d$ for some integer $d > 0$. We require the policy parametrization to be continuously differentiable with respect to $\theta$ (for every state and action). This way, we simultaneously cover the direct parametrization (for which~\cref{cor:dual-discrete} directly provides a policy update) as well as commonly used policy parametrizations (e.g., softmax and the Gaussian policies). Accordingly, the dual problem~(\ref{pb:dual}) can be reformulated as follows for every $\theta \in \mathbb{R}^d$: 
\begin{equation}
\label{eq:algo-objective}
\min_{\lambda \geq 0}  G(\lambda, \theta)
    \coloneqq
    \lambda\varepsilon + \int_{\statespace} \int_{\actionspace} \max_{a' \in \actionspace}\{\advantagetheta(s,a') - \lambda c(a,a')\}\d\pi_{\theta}(a|s) \d\visitationfreqtheta(s).
\end{equation}

\begin{multicols}{2}
Given a current policy represented by the vector~$\theta$, we first solve the one-dimensional convex problem~(\ref{eq:algo-objective}) to obtain its solution~$\lambda^*$. Then, we use the optimal dual multiplier $\lambda^*$ to derive the optimal policy update within the trust region. The procedure is summarized in Algorithm~\ref{algo:policy-update}.
\atmargin{}{Formatted badly!}
Depending on the parametrization of the policy, the steps of~\cref{algo:policy-update-c-wtrpo} are as follows: 

\columnbreak
\begin{algorithm}[H]
\label{algo:policy-update-c-wtrpo}
\caption{\textsc{OT-TRPO}.}
\begin{algorithmic}[1]\label{algo:policy-update}
\STATE Initialize $\pi_{\theta_0}$
\FORALL{$t = 0, 1, \ldots$}
\STATE Estimate $A^{\pi_{\theta_t}}$ and $\rho_{\pi_{\theta_t}}$.
\STATE Compute $\lambda^*\in \text{argmin}_{\lambda \geq 0} G(\lambda, \theta_t)$.
\STATE Update $\theta_{t}$ to $\theta_{t+1}$ using $\lambda^\ast$.
\ENDFOR
\end{algorithmic}
\end{algorithm}
\end{multicols}
\looseness -1
\noindent\textbf{Algorithm~\ref{algo:policy-update} - step 3.} In the discrete setting, the visitation frequency is estimated via Monte Carlo methods. In the continuous case, we weight every visited state equally. We propose three ways to estimate the unknown advantage function via samples\footnote{In the experiments reported in the main paper we used the first and the second method for discrete and continuous environments, respectively. In~\cref{app:details-policy-update} we further comment on the different methods.}:
\begin{enumerate}[leftmargin = *]
\item Monte Carlo methods or TD-learning (for discrete settings only);
\item \gls*{acr:gae}~\cite{Schulman2016}, using a neural network to approximate the value function like in standard actor-critic methods; and
\item Direct estimation via non-linear approximators (e.g., using directly a neural network for the advantage function).
\end{enumerate}

\noindent\textbf{Algorithm~\ref{algo:policy-update} - step 4 (evaluation of $G$).} 
Depending on the setting, we propose various ways to evaluate $G(\lambda,\theta)$. They all apply to both continuous and discrete states. 
\begin{enumerate}[leftmargin = *]
    \item\emph{Finite actions}: Since the maximization in~\eqref{eq:algo-objective} is over finitely many actions, we can directly evaluate~\eqref{eq:algo-objective} for any $\lambda\geq 0$.
    
    \item\emph{Gaussian policy parametrization}: 
    With $m(s)$ being the mean of the Gaussian policy (with fixed variance), we approximate $G(\lambda,\theta)$ by 
    \begin{equation*}
        G(\lambda,\theta)
        \approx
        \begin{cases}
            \lambda\varepsilon+\sum_{s\in\hat\statespace}\max_{a' \in \{a,m(s)\}}\{\advantagetheta(s,a')-\lambda c(m_\theta(s),a')\} \visitationfreqtheta(s)
            & \text{if $\advantagetheta$ via \gls*{acr:gae}}, 
            \\
            \lambda\varepsilon+\sum_{s\in\hat\statespace}\max_{a' \in \hat\actionspace(s)}\{\advantagetheta(s,a')-\lambda c(m_\theta (s),a')\}\visitationfreqtheta(s)
            & \text{if $\advantagetheta$ via NN},
        \end{cases}
    \end{equation*}
    where $\hat\statespace$ are the state visited in the trajectory and $\hat\actionspace(s)$ is a (possibly state-dependent) collection of samples collected from $\actionspace$. 
    
    \item\emph{Arbitrary policy parametrization}: 
    For a neural network approximation of the advantage function, we approximate $G(\lambda,\theta)$ by 
    \begin{equation*}
        \textstyle 
        G(\lambda,\theta)
        \approx
        \lambda\varepsilon+\sum_{s\in\hat\statespace}\sum_{a\in\hat\actionspace_1(s)}\max_{a' \in \hat\actionspace_2(s)}\{\advantagetheta(s,a')-\lambda c(a,a')\}\pi_{\theta}(a|s)\visitationfreqtheta(s),
    \end{equation*}
    where $\hat\statespace$ are the state visited in the trajectory and $\hat\actionspace_i(s)$ are (possibly state-dependent) collections of samples collected from $\actionspace$. 
\end{enumerate}

\noindent\textbf{Algorithm~\ref{algo:policy-update} - step 4 (solving for $\lambda^\ast$).}
Since~\eqref{pb:dual} is a one-dimensional convex optimization problem, $\lambda^\ast$ can be found using any solver for convex optimization problems.

\noindent\textbf{Algorithm~\ref{algo:policy-update} - step 5.}
Update the parameter vector $\theta$.
\begin{enumerate}[leftmargin = *]
    \item\emph{Direct parametrization (finite spaces)}: Update $\theta$ according to~\eqref{eq:opt-policy-discrete} and~\eqref{eq:bij*}.
    
    \item\emph{Direct parametrization via policy network (continuous states, discrete actions)}: Use~\eqref{eq:opt-policy-discrete} and~\eqref{eq:bij*} to compute the optimal policy update at the visited states, denoted by $\pi^\ast_{\theta_t}$. Then, update the policy network by performing gradient descent on the loss $L(\theta)=\sum_{s\in\hat\statespace}\rho_{\pi_{\theta_t}}(s)\norm{\pi_\theta(\cdot|s)-\pi^\ast_{\theta_t}(\cdot|s)}^2$ to steer $\pi_{\theta}$ towards the optimal policy update $\pi_{\theta_t}^\ast$ \emph{within} the trust region. 
    
    \item\emph{Arbitrary policy parametrization}: 
    Since there are infinitely many actions, the computation of the maximization is computationally demanding, and so~\cref{cor:optimal-policy} cannot be directly utilized for the policy update. Yet, we can update the policy via gradient ascent.
    The intuition is as follows: according to~\cref{cor:optimal-policy}, the optimal policy update attains the maximum $\max_{a'\in\actionspace}\{A^{\pi_{\theta_t}}(s,a')-\lambda^\ast c(a,a')\}$ at each state. Thus, we can steer the policy $\pi_\theta$ to maximize
    \begin{equation*}
        \textstyle
        \theta
        \mapsto
        \sum_{s\in\hat\statespace}\int_{\actionspace}\max_{a'\in\actionspace}\{A^{\pi_{\theta_t}}(s,a')-\lambda^\ast c(a_\theta,a')\}\d\pi_{\theta}(a_\theta|s) \rho_{\pi_{\theta_t}}(s).
    \end{equation*}
    In the particular case of a Gaussian policy with parametrized mean (and fixed variance), combined with \gls*{acr:gae} estimate of the advantage function, one can maximize
    \begin{equation*}
        \textstyle
        \theta
        \mapsto
        \sum_{s\in\hat\statespace}\max_{}\{A^{\pi_{\theta_t}}(s,a')-\lambda^\ast c(m_\theta(s),a'),0\} \rho_{\pi_{\theta_t}}(s).
    \end{equation*}
    Intuitively, this update implicitly estimates the transport maps $\optmapuppername{\lambda^\ast}$ and $\optmaplowername{\lambda^\ast}$, which are needed for the optimal policy update. This way, we steer the policy network towards the optimal policy update \emph{within} the trust region. Among others, this policy update allows for the following interpretation: imposing an optimal transport-based trust region constraints is, at least formally, equivalent to maximizing a \emph{regularized} advantage function, where the value of the regularization $\lambda^\ast\geq 0$ is based on the transport cost $c$ and the radius of the trust region $\varepsilon$.
\end{enumerate}

\section{Experiments and Insights}\label{section:experiments}
In this section, we evaluate the performance of \gls*{acr:ottrpo}
across a variety of environments~\cite{Brockman2016,Todorov2012} of increasing complexity, ranging from discrete to continuous settings. We compare it to the classical \gls*{acr:trpo}~\cite{Schulman2015,Hill2018} and \gls*{acr:ppo}~\cite{Schulman2017,Hill2018}, with A2C~\cite{Mnih2016,Hill2018}, with the recent approaches leveraging the Wasserstein distance, \gls*{acr:bgpg}~\cite{Pacchiano2020} and \gls*{acr:wnpg}~\cite{Moskovitz2020} (in continuous settings), and with \gls*{acr:wpo}~\cite{Song2022} (in discrete settings). 
The training curves are shown in \cref{fig:experiments:allenvs}; see~\cref{app:implementation} for implementation details,~\cref{app:details-experiments} for further details on the experimental results, and~\cref{app:ablation} for an ablation study on our algorithm.

\begin{figure}[h]
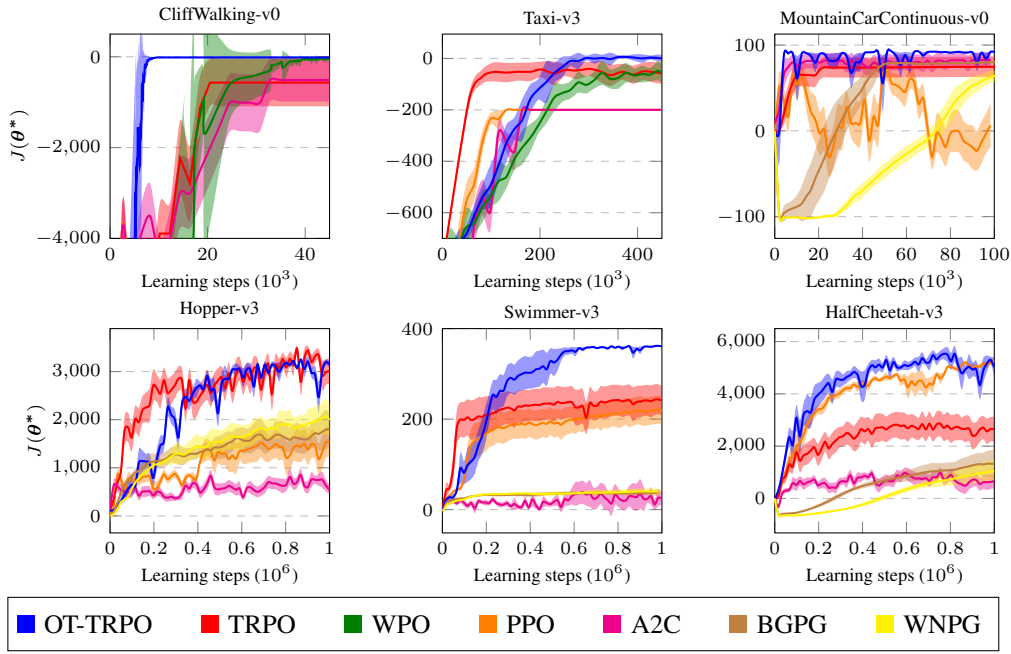

\centering
\pgfplotstableread[col sep = comma]{figures/CliffWalking/PPO.csv}\ppoCliff
\pgfplotstableread[col sep = comma]{figures/CliffWalking/TRPO.csv}\trpoCliff
\pgfplotstableread[col sep = comma]{figures/CliffWalking/WPO.csv}\wpoCliff
\pgfplotstableread[col sep = comma]{figures/CliffWalking/WTRPO.csv}\wtrpoCliff
\pgfplotstableread[col sep = comma]{figures/CliffWalking/A2C.csv}\acCliff
\pgfplotstableread[col sep = comma]{figures/Taxi/PPO.csv}\ppoTaxi
\pgfplotstableread[col sep = comma]{figures/Taxi/TRPO.csv}\trpoTaxi
\pgfplotstableread[col sep = comma]{figures/Taxi/WPO.csv}\wpoTaxi
\pgfplotstableread[col sep = comma]{figures/Taxi/WTRPO.csv}\wtrpoTaxi
\pgfplotstableread[col sep = comma]{figures/Taxi/A2C.csv}\acTaxi
\pgfplotstableread[col sep = comma]{figures/MountainCarContinuous/PPO.csv}\ppoMC
\pgfplotstableread[col sep = comma]{figures/MountainCarContinuous/TRPO.csv}\trpoMC
\pgfplotstableread[col sep = comma]{figures/MountainCarContinuous/WTRPO.csv}\wtrpoMC
\pgfplotstableread[col sep = comma]{figures/MountainCarContinuous/A2C.csv}\acMC
\pgfplotstableread[col sep = comma]{figures/MountainCarContinuous/WNPG.csv}\wnpgMC
\pgfplotstableread[col sep = comma]{figures/MountainCarContinuous/BGPG.csv}\bgpgMC
\pgfplotstableread[col sep = comma]{figures/Hopper/PPO.csv}\ppoHopper
\pgfplotstableread[col sep = comma]{figures/Hopper/TRPO.csv}\trpoHopper
\pgfplotstableread[col sep = comma]{figures/Hopper/WTRPO.csv}\wtrpoHopper
\pgfplotstableread[col sep = comma]{figures/Hopper/A2C.csv}\acHopper
\pgfplotstableread[col sep = comma]{figures/Hopper/WNPG.csv}\wnpgHopper
\pgfplotstableread[col sep = comma]{figures/Hopper/BGPG.csv}\bgpgHopper
\pgfplotstableread[col sep = comma]{figures/Swimmer/PPO.csv}\ppoSwimmer
\pgfplotstableread[col sep = comma]{figures/Swimmer/TRPO.csv}\trpoSwimmer
\pgfplotstableread[col sep = comma]{figures/Swimmer/WTRPO.csv}\wtrpoSwimmer
\pgfplotstableread[col sep = comma]{figures/Swimmer/A2C.csv}\acSwimmer
\pgfplotstableread[col sep = comma]{figures/Swimmer/WNPG.csv}\wnpgSwimmer
\pgfplotstableread[col sep = comma]{figures/Swimmer/BGPG.csv}\bgpgSwimmer
\pgfplotstableread[col sep = comma]{figures/HalfCheetah/PPO.csv}\ppoHalf
\pgfplotstableread[col sep = comma]{figures/HalfCheetah/TRPO.csv}\trpoHalf
\pgfplotstableread[col sep = comma]{figures/HalfCheetah/WTRPO.csv}\wtrpoHalf
\pgfplotstableread[col sep = comma]{figures/HalfCheetah/A2C.csv}\acHalf
\pgfplotstableread[col sep = comma]{figures/HalfCheetah/WNPG.csv}\wnpgHalf
\pgfplotstableread[col sep = comma]{figures/HalfCheetah/BGPG.csv}\bgpgHalf
    
\begin{tikzpicture}[every mark/.append style={mark size=.25pt}]
\begin{groupplot}[group style={group size=3 by 2,
                               horizontal sep=1.5cm,
                               vertical sep=1.2cm,
                               group name=myplot},
                  width=4.5cm,
                  height=4.3cm,]
\input{figures/CliffWalking/tikz.tex}
\input{figures/Taxi/tikz.tex}
\input{figures/MountainCarContinuous/tikz.tex}
\input{figures/Hopper/tikz.tex}
\input{figures/Swimmer/tikz.tex}
\input{figures/HalfCheetah/tikz.tex}
\end{groupplot}
\end{tikzpicture}
\begin{tikzpicture}
\matrix [draw, column sep=15] {
  \node [fill=\colorwtrpo,minimum width=2,label=right:\gls{acr:ottrpo}] {};
  &
  \node [fill=\colortrpo,minimum width=2,label=right:\gls{acr:trpo}] {};
  &
  \node [fill=\colorwpo,minimum width=2,label=right:\gls{acr:wpo}] {};
  &
  \node [fill=\colorppo,minimum width=2,label=right:\gls{acr:ppo}] {};
  &
  \node [fill=\colorac,minimum width=2,label=right:A2C] {};
  &
  \node [fill=\colorbgpg,minimum width=2,label=right:\gls{acr:bgpg}] {};
  &
  \node [fill=\colorwnpg,minimum width=2,label=right:\gls{acr:wnpg}] {};\\
};
\end{tikzpicture}
\vspace{-.2cm}
\caption{Cumulative rewards during the training process in different environments. The shaded area represents the mean $\pm$ the standard deviation across $10$ independent runs. Every policy evaluation in each run is averaged over $10$ sampled trajectories.}
\label{fig:experiments:allenvs}
\end{figure}

Our approach is shown to consistently improve over the other algorithms: \gls*{acr:ottrpo} leads to larger final returns, with lower variance, and only in few cases at the expense of a slightly slower learning curve. 
Four remarks are in order. 
First, the performance gain of \gls*{acr:ottrpo} compared to \gls*{acr:bgpg} and \gls*{acr:wnpg} confirms that trust regions help stabilize training, as already observed in \cite{Schulman2015}.
Second, optimal transport discrepancies induce a more natural notion of ``closeness'' between policies compared to the \gls*{acr:kl} divergence (e.g., see~\cite[Example 2.1]{Aolaritei2022}).
For instance, in CliffWalking-v0, consider the optimal policy $\pi^\ast$ and the candidate policy $\pi$ depicted below, which differ at one state only (see figure). 
\begin{wrapfigure}{r}{0.35\textwidth}
\centering
\begin{tikzpicture}
    \draw[step=0.5,black,thin] (0.0,0.0) grid (4.5,1.5);
    \draw[fill=white] (0.5,0) rectangle (4,0.5);
    \draw[fill=black,opacity=0.3] (0.5,0) rectangle (4,0.5);
    \draw[->,red,ultra thick] (0.15,0.25) --++ (0,0.6) --++ (4.1,0) --++ (0,-0.6); 
    \draw[->,blue,ultra thick] (0.35,0.25) --++ (0,0.4) --++ (3.4,0) --++ (0,-0.4); 
    \node[anchor=west] at (-0.1,1.8) {\footnotesize \textcolor{red}{$\pi^\ast$: optimal policy}, \textcolor{blue}{$\pi$: ``close'' policy}};
\end{tikzpicture}    
\end{wrapfigure}
The optimal transport discrepancy between $\pi$ and $\pi^\ast$ is $\visitationfreq(s)C(\pi(\cdot|s),\pi^\ast(\cdot|s))=\visitationfreq(s) c(\text{Down},\text{Right})$. When using the~\gls*{acr:kl} divergence, instead, the discrepancy is infinite, since the two policies do not share the same support. 
In particular, if initialized with $\pi$, \gls*{acr:trpo} cannot converge to the optimal policy, regardless of the radius of the trust region. 
Third,~\gls*{acr:ottrpo} improves on~\gls*{acr:wpo}, in two ways: (i) it leads to superior performances of the trained policies and (ii) it does not violate the trust region constraint (which, e.g., in Taxi-v3 is the case for $72\%$ of the updates of~\gls*{acr:wpo}). This performance improvement results from the ``mass splitting'' (see~\cref{rem:comparison-song-2}), which is the only difference between the two algorithms (in discrete settings). 
Fourth, in continuous settings, the performance of~\gls*{acr:ottrpo} is aligned with the best-performing alternative approach. In the environment Swimmer-v3, it even yields an improvement of more than 50\% in the performance of the trained agent.

\section{Conclusion and Future Work}
We studied trust region policy optimization for continuous state-action spaces whereby the trust region is defined in terms of a general optimal transport discrepancy.
Our analysis bases on a one-dimensional convex dual reformulation of the optimization problem for the policy update which (i) enjoys strong duality and (ii) directly characterizes the optimal policy update, bypassing the computational burden of evaluating optimal transport discrepancies.
Moreover, we show that the policy update can yield a monotonic improvement of the performance index.
Empowered by our theoretic results, we propose a novel algorithm, \gls*{acr:ottrpo}, for trust region policy optimization with optimal transport discrepancies. We evaluate its performance across several environments. Our results reveal that trust regions defined by optimal transport discrepancies can offer advantages over the~\gls*{acr:kl} divergence or non-trust region methods. 

There are several research directions that merit further investigation. We highlight two.
First, transport costs provide us actionable knobs to shape the geometry of the trust region, and can be used to encode prior knowledge on the environment or preferred exploration strategies. 
Second, we would like to study the convergence properties of the proposed algorithm. 
 
\section*{Acknowledgements}

This project has received funding from Google Brain, Swiss National Science Foundation under the NCCR Automation (grant agreement 51NF40\_180545), and it was partially supported by the ETH AI Center.

\bibliographystyle{plain}
\bibliography{biblio}

\section*{Checklist}


\begin{enumerate}[leftmargin=*]

\item For all authors...
\begin{enumerate}
   \item Do the main claims made in the abstract and introduction accurately reflect the paper's contributions and scope?
   \answerYes{All claims are in line with the paper and its contributions.}
   \item Did you describe the limitations of your work?
   \answerYes{We discuss both benefits and drawbacks of our methodology.}
   \item Did you discuss any potential negative societal impacts of your work?
   \answerNA{We do not believe our work  entails  potential negative societal impacts.}
   \item Have you read the ethics review guidelines and ensured that your paper conforms to them?
    \answerYes{We carefully read ethics review guidelines.}
\end{enumerate}

\item If you are including theoretical results...
\begin{enumerate}
    \item Did you state the full set of assumptions of all theoretical results?
    \answerYes{All assumptions are stated before the corresponding theoretic results.}
    \item Did you include complete proofs of all theoretical results?
    \answerYes{All proofs are included in the supplementary material.}
\end{enumerate}

\item If you ran experiments...
\begin{enumerate}
    \item Did you include the code, data, and instructions needed to reproduce the main experimental results (either in the supplemental material or as a URL)?
    \answerYes{Our code, with  instructions, is  part of the supplementary material.}
    \item Did you specify all the training details (e.g., data splits, hyperparameters, how they were chosen)?
    \answerYes{All training details are specified in the supplementary material.}
    \item Did you report error bars (e.g., with respect to the random seed after running experiments multiple times)?
    \answerYes{Our plots include error bars (mean $\pm$ standard deviation).}
    \item Did you include the total amount of compute and the type of resources used (e.g., type of GPUs, internal cluster, or cloud provider)?
    \answerYes{We comment on the computational time in the supplementary material.}
\end{enumerate}

\item If you are using existing assets (e.g., code, data, models) or curating/releasing new assets...
\begin{enumerate}
    \item If your work uses existing assets, did you cite the creators?
    \answerYes{We cited the creators of OpenAI, Mujoco (which we used for experiments), stable baselines (which provided algorithms), and the creators of state-of-the-art algorithms.}
    \item Did you mention the license of the assets?
    \answerNA{The license of the assets is easily verifiable following its reference.}
    \item Did you include any new assets either in the supplemental material or as a URL?
    \answerYes{Our code is included in the supplementary material.}
    \item Did you discuss whether and how consent was obtained from people whose data you're using/curating?
    \answerNA{All assets are publicly available or asked to the authors.}
    \item Did you discuss whether the data you are using/curating contains personally identifiable information or offensive content?
    \answerNA{The data we used does not contain personal information or offensive content.}
\end{enumerate}

\item If you used crowdsourcing or conducted research with human subjects...
\begin{enumerate}
    \item Did you include the full text of instructions given to participants and screenshots, if applicable?
    \answerNA{We did not use crowdsourcing or conducted research with human subjects.} 
    \item Did you describe any potential participant risks, with links to Institutional Review Board (IRB) approvals, if applicable?
    \answerNA{We did not use crowdsourcing or conducted research with human subjects.} 
    \item Did you include the estimated hourly wage paid to participants and the total amount spent on participant compensation?
    \answerNA{We did not use crowdsourcing or conducted research with human subjects.} 
\end{enumerate}

\end{enumerate}


\newpage
\appendix

\doparttoc
\faketableofcontents
\part*{Appendix}
\parttoc

\section{Appendix}

\subsection{More details on related work}\label{app:related-work}
Optimal transport and in particular the Wasserstein distance has found various applications in \gls{acr:rl}, despite the computational challenge raised by its evaluation.

\paragraph{Optimal transport for \gls{acr:po}.}
In~\cite{Richemond2017}, the authors established a connection between policy gradient in Wasserstein trust regions and variational optimal transport, suggesting to solve the Fokker-Planck equations and to use diffusion processes.
Concurrently, \cite{zhang-et-al2018} formulated the~\gls{acr:po} problem as a gradient descent flow on the space of probability measures using Wasserstein Gradient Flows~\cite{Ambrosio2008} which is then solved approximately via particles.

\paragraph{\gls{acr:bgpg}~\cite{Pacchiano2020}.}
In~\cite{Pacchiano2020}, the authors proposed to replace the \gls*{acr:kl} divergence trust region from \gls*{acr:trpo}~\cite{Schulman2015} by a Wasserstein distance penalty in a behavioral space; the proposed approach embeds the policies in a latent behavioral space via a map acting on the trajectories induced by the policies, and leverages the Wasserstein distance to compare two embeddings. To approximate the Wasserstein distance, they exploit the dual formulation of the entropy-regularized Wasserstein distance. 
Although our alternating procedure in~\cref{sec:practical-ottrpo} may formally resemble the approach of~\cite{Pacchiano2020} in spirit, our approach (i) builds on the idea of trust regions and (ii) does not consider policy embeddings or entropy regularization. Not surprisingly, our~\cref{algo:policy-update} is fundamentally different from Algorithms 1 and 3 in \cite{Pacchiano2020}. If \gls{acr:bgpg} can be seen as a variant of \gls{acr:trpo} with a Wasserstein regularization, we propose a novel distinct variant of OT-based \gls{acr:trpo} fully based on trust regions. 

\paragraph{\gls{acr:wnpg}~\cite{Moskovitz2020}.}
While \cite{Pacchiano2020} proposed to use the Wasserstein distance as a global penalty to the objective, \cite{Moskovitz2020} further suggested to incorporate additional information about the local behavior of policies encapsulated in the so-called Wasserstein Information Matrix. They proposed accordingly \gls{acr:wnpg} to speed up policy optimization using a Wasserstein natural gradient.
The cornerstone of this algorithm is the estimation of a Wasserstein natural gradient stemming from a second-order expansion of the 2-Wasserstein distance between two (parametric) behavioral embedding distributions of two parameterized policies.    

\paragraph{\gls{acr:wpo}~\cite{Song2022}.}
The closest related work to ours is the recent paper~\cite{Song2022} which studied~\gls{acr:po} with Wasserstein distance and Sinkhorn divergence-based trust regions for discrete (and finite) action spaces. In contrast to~\cite{Song2022}, the present work addresses the general setting of compact subsets of Polish spaces encompassing the cases of continuous and discrete state-action spaces as particular cases. While we also adopt a duality approach, our level of generality induces many challenges compared to the discrete action space setting (see~\cref{rem:comparison-song-1}).
Moreover, the methods proposed in~\cite{Song2022} do not allow to restrict the set of policies to a particular family of distributions, which might lead to very large models.
Conversely, the novel algorithm we propose can handle policy parametrization (including direct parametrization). We refer the reader to~\cref{sec:ot-trpo} for more detailed comparisons with~\cite{Song2022}.

\subsection{Comments on ``mass splitting''} \label{appendix:mass-splitting}

In this section, we elaborate on the concept of ``mass splitting''. Specifically, we propose two examples showing that splitting the probability mass is in general required to construct optimal policy updates.

\begin{example}\label{example:wrong}
Let $\statespace = \{s_1\}$, $\actionspace = \{a_1, a_2\}$, and~$\varepsilon\in(0,1)$. Suppose $c(a_1,a_2) = 1,c(a_1,a_1)=c(a_2,a_2)=0$, $\pi(\cdot|s_1) = \diracdelta{a_1}$, $A^{\pi}(s_1,a_1) = 0$ and~$A^{\pi}(s_1,a_2) = 1$. Clearly, it is optimal to assigns as much probability mass as possible to action $a_2$; i.e., $\tilde{\pi}(\cdot|s_1) = (1-\varepsilon)\diracdelta{a_1}+\varepsilon\diracdelta{a_2}$, which corresponds to $t^\ast = 1-\varepsilon$, $\actionsamplenewlower{1}{1} = a_1$ and $\actionsamplenewlower{2}{1} = a_2$ using the notation of Corollary~\ref{cor:dual-discrete}.
However, if an optimal policy was to be described by a single transport map (as in~\cite[ Theorem~1, Eq.~(5), p.~4]{Song2022}), the only possible solutions would be~$\pi_1(\cdot|s_1) = \diracdelta{a_1}$ and~$\pi_2(\cdot|s_1) = \diracdelta{a_2}$, which are respectively sub-optimal and infeasible.
\end{example}

\begin{example}\label{ex:mass-split}
Consider an agent in an initial state $s_0$ who can move left (L) or right (R). The rewards are $r(s_0,\mathrm L)=-1$, and $r(s_0,\mathrm R)=+1$, and the task terminates whenever the agents reaches $s_1$ or at $s_2$, as shown below.
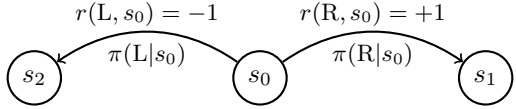
\begin{center}
\begin{tikzpicture}
    \node[circle,draw,thick] at (0,0) (s0) {$s_0$};
    \node[circle,draw,thick] at (3,0) (s1) {$s_1$};
    \node[circle,draw,thick] at (-3,0) (s2) {$s_2$};
    
    \draw[->,thick] (s0) to[out=35,in=145] (s1);
    \node at (1.5,0.3) {\footnotesize $\pi(\mathrm R|s_0)$};
    \node at (1.5,0.85) {\footnotesize $r(\mathrm R,s_0)=+1$};
    
    \draw[->,thick] (s0) to[out=145,in=35] (s2);
    \node at (-1.5,0.3) {\footnotesize $\pi(\mathrm L|s_0)$};
    \node at (-1.5,0.85) {\footnotesize $r(\mathrm L,s_0)=-1$};
\end{tikzpicture}
\end{center}
Consider the initial policy $\pi_0(\mathrm{L}|s_0)=1$ and $\pi_0(\mathrm{R}|s_0)=0$, and trust region defined by the binary distance, and let $\varepsilon\in(0,1)$.
Then, \cref{cor:dual-discrete} yields the new policy $\pi_1(\mathrm L|s)=1-\varepsilon$ and $\pi_1(\mathrm R|s)=\varepsilon$. Note that this update requires ``mass splitting''. A second update yields $\pi_2(\mathrm L|s)=1-2\varepsilon$ and $\pi_2(\mathrm R|s)=2\varepsilon$, and after $1/\varepsilon$ the routine converges to the optimal policy $\pi^\ast(\mathrm L|s)=0$ and $\pi^\ast(\mathrm R|s)=1$.
Conversely, the update of~\cite{Song2022} (i.e., \emph{without} ``mass splitting'') leads to either $\pi_1=\pi_0$, which is suboptimal, or to $\pi_1=\pi^\ast$, which violates the trust region constraint. 
\end{example}

\subsection{Implementation details}\label{app:implementation}
In this section, we present the implementation details supporting our experimental results. 

\subsubsection{\texorpdfstring{\gls{acr:ottrpo}}{Proposed algorithm} details}
We now discuss the details concerning the implementation and tuning of our algorithm. Our code,  along with the instructions to set up the environment and run it, is available at \url{https://gitlab.ethz.ch/lnicolas/ot-trpo}.

\paragraph{Discrete settings.}
For the discrete settings, we developed a custom training and testing framework. We used TD-learning to estimate the advantage function, with learning rate $\alpha$ and discount factor $\gamma$. Namely, we first collect a set $\mathcal{T}$ of trajectory. Then, we initialize $Q(s,a) = 0$ for all $(s,a) \in \statespace\times\actionspace$ and for all $\tau \in \mathcal{T}$ and every $(s_t, a_t, r_t, s_{tt}, a_{tt}) \in \tau$ we update $Q$ as
\[
Q(s_t, a_t) = (1 - \alpha) Q(s_t, a_t) + \alpha (r_t + \gamma Q(s_{tt}, a_{tt})).
\]
The transport cost used is the binary distance\footnote{As a result, the optimal transport discrepancy corresponds to the total variation between the probability distributions over the actions.} (i.e., $c(a,a')=0$ if $a=a'$ and $c(a,a')=1$ otherwise), and the trust region radius is $\varepsilon$. Every $n$ full environment simulations the policy is updated. The parameters for the different environments are reported in \cref{tab:discrete:hyperparameters}.

\begin{table}[H]
    \centering
    \begin{tabular}{l r r}
    \toprule
        Parameter & \textit{CliffWalking-v0} & \textit{Taxi-v3}\\\midrule
        $\alpha$ & $0.999999$ & $0.9$\\
        $\gamma$ & $0.2$ & $0.5$\\
        $\varepsilon$ & $0.01$ & $0.01$\\
        $n$ & $1$ & $32$\\\bottomrule
    \end{tabular}
    \caption{Hyperparameters for the discrete environments.}
    \label{tab:discrete:hyperparameters}
\end{table}

With respect to the computational complexity of the training process, one can assess the following about the main steps of \cref{algo:policy-update}:
\begin{enumerate}
    \item[\textbf{Step 3}.] The complexity of TD-learning is linear in the number $N$ of samples collected during the rollout of the current policy; i.e., $\mathcal{O}(N)$.
    \item[\textbf{Step 4}.] This step boils down to \texttt{scipy.optimize.minimize\_scalar}\footnote{\url{https://scipy.org}}. The complexity of each function call can be shown to be $\mathcal{O}(|\mathcal{S}||\mathcal{A}|^2)$.
    \item[\textbf{Step 5}.] The complexity of this step can be shown to be $\mathcal{O}(|\mathcal{S}||\mathcal{A}|^2)$.
\end{enumerate}

\paragraph{Continuous settings.}
For the continuous settings, we developed an agent which can be interfaced with stable baselines~\cite{Hill2018}. In the benchmark provided, we estimate the advantage function via the \gls{acr:gae}, with coefficient \texttt{gae\_lambda}. We use the policy and value network provided in the \texttt{ActorCriticPolicy} class from \cite{Hill2018}, with network sizes $[64, 64]$ and the activation \texttt{activation\_fn}.
Accordingly, the policy is a Gaussian policy with fixed standard deviation whose mean is expressed by the policy network. That is, we do not use a neural network to approximate the advantage function. We perform stochastic gradient descent in batches of size \texttt{batch\_size}, every \texttt{n\_steps} timesteps, for \texttt{n\_epochs} epochs, with learning rate \texttt{learning\_rate}, and with maximum gradient norm \texttt{max\_grad\_norm}. 
The value function loss is multiplied by \texttt{vf\_coef}. The transport cost is the square Euclidean distance\footnote{Note that the corresponding optimal transport discrepancy is not a distance (but itself a square distance).}, and the trust region radius is $\varepsilon$. Along the lines of the implementation of \gls{acr:ppo}~\cite{Schulman2017,Hill2018}, we introduce a Z-normalization of the advantage estimates if \texttt{normalize} is \texttt{True}, a state dependent exploration with sample frequency \texttt{sde\_sample\_freq} if this value is not $-1$, and an orthogonal initialization if \texttt{ortho\_init} is \texttt{True}. The logarithm of the standard deviation of the Gaussian policy network is initialized to \texttt{log\_std\_init}. The parameters for the different environments are reported in \cref{tab:continuous:hyperparameters}.

\begin{table}[H]
    \centering
    \begin{tabular}{l r r r r}
    \toprule
    Parameter & \textit{MountainCarCont.-v0} & \textit{Hopper-v3} & \textit{Swimmer-v3} & \textit{HalfCheetah-v3}\\\midrule 
    $\varepsilon$ & $8.9919$ & $0.4$ & $0.2$ & $0.0548$\\
    \texttt{n\_steps} & $512$ & $512$ & $1024$ & $1024$\\
    \texttt{batch\_size} & $256$ & $512$ & $64$ & $256$\\
    \texttt{n\_epochs} & $10$ & $10$ & $4$ & $20$\\
    \texttt{learning\_rate} & $0.0029$ & $0.0008$ & $0.0003$ & $0.0003$\\
    \texttt{max\_grad\_norm} & $0.7$ & $0.1$ & $0.5$ & $0.8$\\
    \texttt{activation\_fn} & \texttt{ReLU} & \texttt{Tanh} & \texttt{Tanh} & \texttt{LeakyReLU}\\
    \texttt{vf\_coef} & $0.6143$ & $0.6349$ &  $0.5$ & $0.0070$\\
    \texttt{gae\_lambda} & $0.95$ & $0.92$ & $0.98$ & $0.9$\\
    \texttt{gamma} & $0.999$ & $0.995$ & $0.999$ & $0.99$\\
    \texttt{normalize} & \texttt{True} & \texttt{True} & \texttt{False} & \texttt{True}\\
    \texttt{sde\_sample\_freq} & $128$ & $16$ & $-1$ & $128$\\
    \texttt{ortho\_init} & \texttt{True} & \texttt{True} & \texttt{True} & \texttt{False} \\
    \texttt{log\_std\_init} & $0.0$ & $-0.3619$ & $0.0$ & $-2.0291$\\\bottomrule
    \end{tabular}
    \caption{Hyperparameters and network sizes for the continuous environments. The values are rounded to the fourth decimal digit.}
    \label{tab:continuous:hyperparameters}
\end{table}

With respect to the computational complexity of the training process, one can assess the following about the main steps of \cref{algo:policy-update}:
\begin{enumerate}
    \item[\textbf{Step 3}.] The complexity of using a \gls{acr:gae} is linear in the number $N$ of samples collected during the rollout of the current policy; i.e., $\mathcal{O}(N)$.
    \item[\textbf{Step 4}.] This step boils down to \texttt{scipy.optimize.minimize\_scalar}. The complexity of each function call can be shown to be $\mathcal{O}(N)$.
    \item[\textbf{Step 5}.] This corresponds to a gradient descent step. Therefore, the computational costs are in line with the algorithms to which we compare. 
\end{enumerate}

\subsubsection{Compared algorithms details}
We compare our algorithm to several baselines, which we implemented and tuned as follows: 
\begin{itemize}[leftmargin = *]
    \item The implementation, hyperparameters, and network sizes for \gls{acr:trpo}~\cite{Schulman2015}, \gls{acr:ppo}~\cite{Schulman2017}, and A2C~\cite{Mnih2016} can be found in stable baselines~\cite{Hill2018}.
    \item We thank the authors of~\cite{Pacchiano2020} for privately providing us the code and tuning for \gls{acr:bgpg}.
    \item For \gls{acr:wnpg}~\cite{Moskovitz2020}, we used the code and tuning publicly available at \url{https://github.com/tedmoskovitz/WNPG}.
    \item We thank the authors of~\cite{Song2022} for pointing us to their implementation (with tuning parameters) of \gls{acr:wpo}, available at \url{https://github.com/efficientwpo/EfficientWPO}.
    This code has inspired in some parts our training and testing framework. Their primal update has been ported in our code to properly compare the only difference in the algorithms (in discrete settings). We use the same parameters for both \gls{acr:wpo} and \gls{acr:ottrpo}.
\end{itemize}
To the best of our knowledge, we used the hyperparameters and network sizes with the best performances for every algorithm in every environment. The results obtained for the baselines considered are comparable or better to the ones found in the literature; e.g., see \cite{Song2022}. We refrained from reporting the results of the algorithms whenever they were not reproducible or not comparable to the others.

For \gls{acr:bgpg}~\cite{Pacchiano2020} and \gls{acr:wnpg}~\cite{Moskovitz2020}, we used \texttt{tensorflow.compat} to ensure compatibility of their code (written with TensorFlow 1.x) with our installation of TensorFlow 2.x.
The code reproduces the results reported in these works, but the comparisons provided in this paper refer to the updated environments; e.g., \textit{Hopper-v3} instead of \textit{Hopper-v2}.

Finally, to the best of our knowledge, the timescale of the results obtained with the code in \url{https://github.com/efficientwpo/EfficientWPO} and the results ported in our work might differ. In the former, the evaluation is performed after a certain number of episodes (full environment simulation), rather than timesteps, and the conversion is not clear from the code (as every episode might differ in length). We instead adhere to the standard provided in~\cite{Hill2018}.

\subsection{Ablation study}\label{app:ablation}
We investigate how different (i) trust region radii $\varepsilon$ and (ii) transportation costs, affect the learning performances in the \emph{Taxi-v3} environment. These effects are shown in \cref{fig:experiments:ablation}.

\newcommand{\colorEpsLarge}{orange}
\newcommand{\colorEpsVeryLarge}{red}
\newcommand{\colorEpsSmall}{grgreen}
\newcommand{\colorEpsRef}{blue}
\newcommand{\colorEpsVerySmall}{magenta}
\newcommand{\colorCostExpensive}{orange}
\newcommand{\colorCostVeryExpensive}{red}
\newcommand{\colorCostCheap}{grgreen}
\newcommand{\colorCostVeryCheap}{magenta}
\newcommand{\colorCostRef}{blue}

\begin{figure}[h]
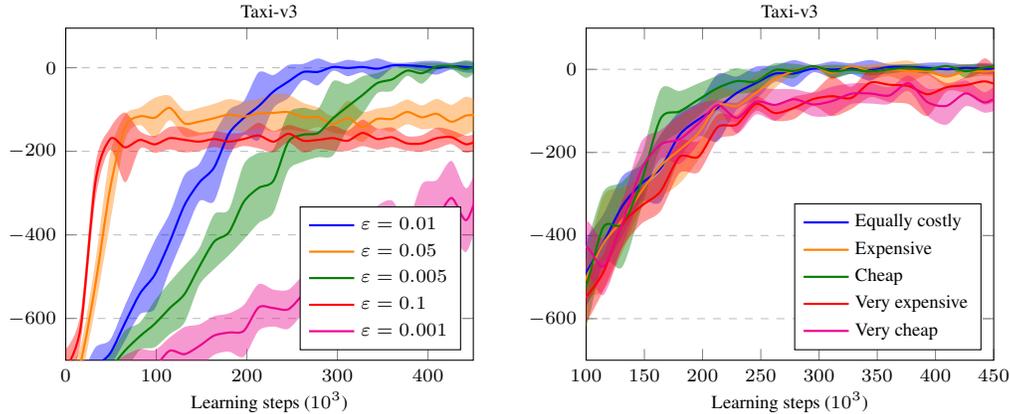

    \pgfplotstableread[col sep = comma]{figures/Ablation/Taxi/Epsilon/Large.csv}\EpsLarge
    \pgfplotstableread[col sep = comma]{figures/Ablation/Taxi/Epsilon/VeryLarge.csv}\EpsVeryLarge
    \pgfplotstableread[col sep = comma]{figures/Ablation/Taxi/Epsilon/Small.csv}\EpsSmall
    \pgfplotstableread[col sep = comma]{figures/Ablation/Taxi/Epsilon/VerySmall.csv}\EpsVerySmall
    \pgfplotstableread[col sep = comma]{figures/Ablation/Taxi/Epsilon/Ref.csv}\EpsRef
    \pgfplotstableread[col sep = comma]{figures/Ablation/Taxi/Cost/Expensive.csv}\CostExpensive
    \pgfplotstableread[col sep = comma]{figures/Ablation/Taxi/Cost/VeryExpensive.csv}\CostVeryExpensive
    \pgfplotstableread[col sep = comma]{figures/Ablation/Taxi/Cost/Cheap.csv}\CostCheap
    \pgfplotstableread[col sep = comma]{figures/Ablation/Taxi/Cost/VeryCheap.csv}\CostVeryCheap
    \pgfplotstableread[col sep = comma]{figures/Ablation/Taxi/Cost/Ref.csv}\CostRef
    \centering
    \begin{tikzpicture}
    \begin{groupplot}[group style={group size=2 by 1,
                               horizontal sep=1.5cm,
                               vertical sep=1.2cm,
                               group name=myplot},
                  width=7cm,
                  height=6cm,]
    \input{figures/Ablation/Taxi/Epsilon/tikz.tex}
    \input{figures/Ablation/Taxi/Cost/tikz.tex}
    \end{groupplot}
    \end{tikzpicture}
    \caption{Ablation study. The shaded area represents the mean $\pm$ the standard deviation across $10$ independent runs. Every policy evaluation in each run is averaged over $10$ sampled trajectories.}
    \label{fig:experiments:ablation}
\end{figure}

\paragraph{Different trust region radii $\varepsilon$.}
As expected, larger radii lead to significantly steeper learning curves, at the expense of converging to suboptimal policies. Conversely, a smaller radius results in slower learning trajectories. Our trust region radius $\varepsilon=0.01$ seems to balance the two effects.

\paragraph{Different transportation costs.}
The \textit{Taxi-v3} environment is characterized by two radically different set of actions: $\mathrm{Move} = \{\mathrm{Up}, \mathrm{Left}, \mathrm{Right}, \mathrm{Down}\}$ and $\mathrm{Passenger} = \{\mathrm{PickUp}, \mathrm{DropOff}\}$.
We use this insight to provide a simple example of how transport costs can affect the geometry of the action space and in turn affect the learning process. 
Specifically, we consider different cost functions. Let $m,m' \in \mathrm{Move}, p,p' \in \mathrm{Passenger}$. Consider now the following transport costs (for all of them we set $c(x, x) = 0 \ \forall x \in \mathrm{Move} \cup \mathrm{Passenger}$): 
\begin{itemize}[leftmargin = *]
    \item[] \textit{Equally costly.} In this cost term, $c(m,m') = c(m',m) = c(p,p') = c(p',p) = c(m,p) = c(p,m)$. This is also a distance.
    \item[] \textit{Cheap [Expensive].} In this cost term, $c(m,m') = c(m',m)$, $c(p,p') = c(p',p)$, but $c(m,p) > c(p,m)$ [$c(m,p) < c(p,m)$]. In particular, $c(p,m) = c(m,p) - 0.2$ [$c(p,m) = c(m,p) + 0.2$]. Since it is not symmetric, this is not a distance.
    \item[] \textit{Very cheap [expensive].} In this cost term, $c(m,m') = c(m',m)$, $c(p,p') = c(p',p)$, but $c(m,p) \gg c(p,m)$ [$c(m,p) \ll c(p,m)$]. In particular, $c(p,m) = c(m,p) - 0.5$ [$c(p,m) = c(m,p) + 0.5$]. Since it is not symmetric, this is not a distance.
\end{itemize}

Notice that a ``local minimum'' in the \textit{Taxi-v3} environment is to never play actions in the $\mathrm{Passenger}$ set: such a strategy yields a score of $-200$. Indeed, performing a $\mathrm{DropOff}$ leads to large penalty when not at the right location. As a result, carrying a passenger is ``risky''. However, these are also the actions that are required to cross the $-200$ bar and learn the optimal policy. We can interpret the aforementioned costs as follows: 
\begin{itemize}[leftmargin = *]
    \item[] \textit{Equally costly.} This cost function ignores the distinction between the two classes of actions: moving from one to the other is equally costly.
    \item[] \textit{Cheap [Expensive].} Moving mass towards ``risky'' actions is incentivized [penalized]. This cost function shapes the landscape of the action space so that there is a downhill [uphill] from $\mathrm{Move}$ to $\mathrm{Passenger}$, and withing these sets the geometry is analogous to the \textit{Equally costly} one.
    \item[] \textit{Very cheap [expensive].} Analogous to \textit{Cheap} [\textit{Expensive}].
\end{itemize}

We observe that with \textit{Very cheap} [\textit{expensive}] the agent learns apparently worse than with the others. Intuitively, this is due to the fact that the \textit{Very cheap} cost incentivizes too much the agent to drop off the passenger; as a result, it completes the task, but it also incurs many illegal $\mathrm{DropOff}$s, which yield the penalty. Oppositely, the \textit{Very expensive} cost is too conservative, and the agent learns to perform too few $\mathrm{DropOff}$s, being unable to complete the task correctly consistently.

On the other hand, the \textit{Cheap} [\textit{Expensive}] cost trains the agent slightly faster [slower] than the \textit{Equally costly} option. Namely, these costs balance in a different way the trade-off between exploration and exploitation.

This example highlights the impact of the geometry of the action space, which, in turn, is another tuning parameter available to practitioners. For instance, one could consider transport costs which embed an exploration penalties (e.g., moving mass from actions that are already rarely used or to others that are almost always used can be penalized), or physical considerations (e.g., an expert supervisor might know that certain actions are most likely the correct ones given certain configurations of the agent).

\paragraph{Sensitivity to the advantage estimation}
Both the dual problem~\eqref{pb:dual} and the ``regularized'' policy loss rely on the accuracy of the advantage function estimation, as well as on the solution of a maximization problem over a continuous space, which is approximated by finitely many evaluated points.
That is, when training on continuous actions spaces, we rely on the approximation
\begin{equation}
    \max_{a' \in \actionspace} \{\advantage(s,a') - \lambda c(a,a')\} \approx \max_{a' \in \actionspace'} \{\hat{\advantage}(s,a') - \lambda c(a,a')\},
\end{equation}
where $\actionspace' \subset \actionspace$ is a (possibly state-dependent) finite set of actions.
While smoothness of the objective guarantees convergence as $|\actionspace'|\rightarrow \infty$ with uniformly chosen actions in $\actionspace$, we empirically investigate the sensitivity of the method to various practical advantage estimation schemes.

In the experimental results so far presented, we use a \gls*{acr:gae} based estimate of the advantage function. 
In continuous action spaces, we explore a neural network approximation for the advantage function as well. 
Overall, this approach does not perform comparably to the single sample estimation of the objective with \gls*{acr:gae}. 
The training performances in the \textit{Hopper-v3} environment for different cardinalities of $\actionspace'$ are juxtaposed in \cref{fig:experiments:advantage-estimation:ablation-n-samples}, and the trained agents scores are summarized in \cref{tab:experiments:advantage-estimation:ablation-n-samples}. In \cref{fig:experiments:advantage-estimation} we compare the performances of the neural network approximation of the advantage function with the \gls*{acr:gae} in the environments \textit{HalfCheetah-v3}, \textit{Hopper-v3}, and \textit{MountainCarContinuous-v0}. The number of sampled actions is $|\actionspace'| = 4$, $|\actionspace'| = 2$, and $|\actionspace'| = 16$, respectively.

\newcommand{\colorGAE}{orange}
\newcommand{\colorNNTwo}{blue}
\newcommand{\colorNNFour}{red}
\newcommand{\colorNNEight}{grgreen}
\newcommand{\colorNNSixteen}{blue}
\newcommand{\colorNNThirtytwo}{magenta}

\begin{figure}[h]
    \pgfplotstableread[col sep = comma]{figures/Hopper/WTRPO.csv}\HopperAdvGAE
    \pgfplotstableread[col sep = comma]{figures/Ablation/Advantage/nActionsSamples/2.csv}\HopperAdvTwo
    \pgfplotstableread[col sep = comma]{figures/Ablation/Advantage/nActionsSamples/4.csv}\HopperAdvFour
    \pgfplotstableread[col sep = comma]{figures/Ablation/Advantage/nActionsSamples/8.csv}\HopperAdvEight
    \pgfplotstableread[col sep = comma]{figures/Ablation/Advantage/nActionsSamples/16.csv}\HopperAdvSixteen
    \pgfplotstableread[col sep = comma]{figures/Ablation/Advantage/nActionsSamples/32.csv}\HopperAdvThirtytwo
    \centering
    \begin{tikzpicture}
    \begin{groupplot}[group style={group size=1 by 1,
                              horizontal sep=1.5cm,
                              vertical sep=1.2cm,
                              group name=myplot},
                  width=10cm,
                  height=6cm,]
    \input{figures/Ablation/Advantage/nActionsSamples/tikz.tex}
    \end{groupplot}
    \end{tikzpicture}
    \caption{Ablation study. The advantage function is estimated either via \gls*{acr:gae} or via a neural network approximation ``NN ($|\actionspace'|$)''. The shaded area represents the mean $\pm$ the standard deviation across $10$ independent runs. Every policy evaluation in each run is averaged over $10$ sampled trajectories.}
    \label{fig:experiments:advantage-estimation:ablation-n-samples}
\end{figure}

\begin{table}[h]
    \centering
    \begin{tabular}{l r}
    \toprule
    Advantage estimator method & Trained agent scores \\\midrule    
   \gls*{acr:gae}  & $2939 \pm 162$ \\
    NN $(2)$     & $ 327 \pm 96 $ \\
    NN $(4)$     & $ 305 \pm 72 $ \\
    NN $(8)$     & $ 452 \pm 136 $\\
    NN $(16)$     & $ 387 \pm 94 $ \\
    NN $(32)$     & $ 497 \pm 150 $ \\\bottomrule
    \end{tabular}
    \caption{Trained agents scores when the advantage function is estimated either via \gls*{acr:gae} or via a neural network approximation ``NN ($|\actionspace'|$)''.}
    \label{tab:experiments:advantage-estimation:ablation-n-samples}
\end{table}

\begin{figure}[h]
    \pgfplotstableread[col sep = comma]{figures/Hopper/WTRPO.csv}\HopperAdvGAE
    \pgfplotstableread[col sep = comma]{figures/Ablation/Advantage/Hopper.csv}\HopperAdvNN
    \pgfplotstableread[col sep = comma]{figures/HalfCheetah/WTRPO.csv}\HalfCheetahAdvGAE
    \pgfplotstableread[col sep = comma]{figures/Ablation/Advantage/HalfCheetah.csv}\HalfCheetahAdvNN
    \pgfplotstableread[col sep = comma]{figures/MountainCarContinuous/WTRPO.csv}\MountainCarAdvGAE
    \pgfplotstableread[col sep = comma]{figures/Ablation/Advantage/MountainCarContinuous.csv}\MountainCarAdvNN
    \centering
    \begin{tikzpicture}
    \begin{groupplot}[group style={group size=3 by 1,
                              horizontal sep=1.5cm,
                              vertical sep=1.2cm,
                              group name=myplot},
                  width=4.5cm,
                  height=4.3cm,]
    \input{figures/Ablation/Advantage/MountainCarContinuousTikz.tex}
    \input{figures/Ablation/Advantage/HopperTikz.tex}
    \input{figures/Ablation/Advantage/HalfCheetahTikz.tex}
    \end{groupplot}
    \end{tikzpicture}
    \begin{tikzpicture}
\matrix [draw, column sep=15] {
  \node [fill=\colorGAE,minimum width=2,label=right:GAE] {};
  &
  \node [fill=\colorNNTwo,minimum width=2,label=right:NN] {};\\
};
\end{tikzpicture}
    \vspace{-.2cm}
    \caption{Ablation study. The advantage function is estimated either via \gls*{acr:gae} or via a neural network approximation. The shaded area represents the mean $\pm$ the standard deviation across $10$ independent runs. Every policy evaluation in each run is averaged over $10$ sampled trajectories.}
    \label{fig:experiments:advantage-estimation}
\end{figure}

In general, we observe that (i) simultaneous training of an advantage network and the \gls{acr:ottrpo} policy update could be unstable, and (ii) in the case of convergence, \gls{acr:ottrpo} with neural advantage function approximation converges to a suboptimal policy.
We hypothesize that the \gls{acr:ottrpo} objective is sensitive to biased value estimates in the neighborhood of the mean of the Gaussian policy network.
It will be subject of future research if and to what extent these results can be improved, for instance via normalizing flow policy parametrizations~\cite{ward2019improving}.

To support our observation regarding the failure mode of biased advantage estimates, we provide an ablation study for different $\lambda$ parameters for the \gls*{acr:gae}. The training curves are shown in \cref{fig:experiments:ablation:lambda-gae}.

\newcommand{\colorLambdaZeroFive}{red}
\newcommand{\colorLambdaZeroNine}{blue}
\newcommand{\colorLambdaOne}{orange}
\begin{figure}
    \centering
    \pgfplotstableread[col sep = comma]{figures/Hopper/WTRPO.csv}\HopperOne
    \pgfplotstableread[col sep = comma]{figures/Ablation/Advantage/lambdaGAE/lambda05.csv}\HopperZeroFive
    \pgfplotstableread[col sep = comma]{figures/Ablation/Advantage/lambdaGAE/lambda09.csv}\HopperZeroNine
    \centering
    \begin{tikzpicture}
    \begin{groupplot}[group style={group size=1 by 1,
                              horizontal sep=1.5cm,
                              vertical sep=1.2cm,
                              group name=myplot},
                  width=10cm,
                  height=6cm,]
    \input{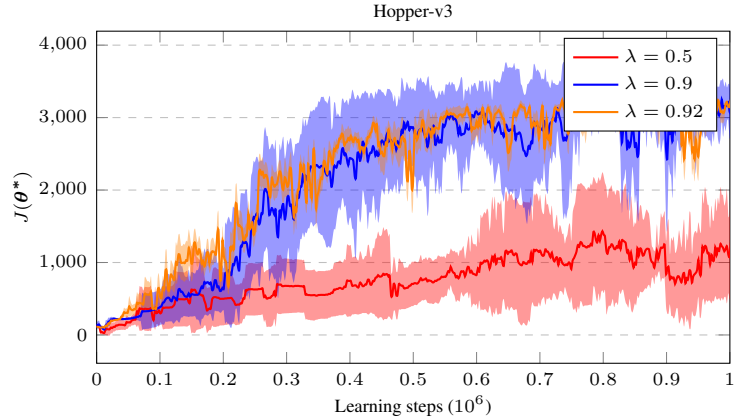}
    \end{groupplot}
    \end{tikzpicture}
    \caption{Ablation study on the $\lambda$ hyperparameter of the \gls*{acr:gae}. The shaded area represents the mean $\pm$ the standard deviation across $10$ independent runs. Every policy evaluation in each run is averaged over $10$ sampled trajectories.}
    \label{fig:experiments:ablation:lambda-gae}
\end{figure}

It is known that for $\lambda=1$, the \gls*{acr:gae} is equivalent to an unbiased Monte Carlo estimate of the advantage~\cite{Schulman2016}.
The observation that state-of-the-art performances are achieved with a \texttt{gae\_lambda} hyperparameter very close to $1$ (i.e., an unbiased Monte Carlo estimate) supports the idea that approximating the maximization with a biased advantage estimator is prone to instability.

\subsection{More details on the policy update}
\label{app:details-policy-update}
We shall now discuss how the implementation affects the exactness guarantees of \cref{thm:dual}. Indeed, albeit theoretically the algorithm presented guarantees an optimal policy update at every step, the practical implementation might resort to approximations. We consider as running example the simple setting in \cref{ex:mass-split}, where all the computations can be carried out analytically. We consider a generic current policy with $\pi_{\theta}(L|s_0) = \theta, \pi_{\theta}(R|s_0) = 1 - \theta$. 

For this example, we can pedagogically solve both \eqref{pb:primal} and \eqref{pb:dual}. The $Q$-function is $Q^{\pi_{\theta}}(s_0, L) = -1, Q^{\pi_{\theta}}(s_0, R) = 1$. Then, the value function at $s_0$ can be expressed in terms of the current policy as 
\[
V^{\pi_{\theta}}(s_0) = \theta Q^{\pi_{\theta}}(s_0, L) + (1 - \theta) Q^{\pi_{\theta}}(s_0, R) = 1 - 2\theta,
\]
and the advantage reads
\[
A^{\pi_{\theta}}(s_0, L) = Q^{\pi_{\theta}}(s_0, L) - V^{\pi_{\theta}}(s_0) = 2(\theta - 1), \quad A^{\pi_{\theta}}(s_0, R) = Q^{\pi_{\theta}}(s_0, R) - V^\pi(s_0) = 2\theta.
\]
Trivially, the visitation frequency is fully described by $\rho_{\pi_{\theta}}(s_0) = 1$; we can ignore the terminal states as they do not affect the problem. We consider the binary transportation cost $c(L, R) = c(R, L) = 1, c(L,L) = c(R,R) = 0$, and we first solve \eqref{pb:primal}. By direct inspection, it is easy to see that there are two cases. Let the new policy be $\pi_{\tilde{\theta}}(L|s_0) = \tilde{\theta}, \pi_{\tilde{\theta}}(R|s_0) = 1 - \tilde{\theta}$. If $\theta < \varepsilon$, then the solution is to move all the mass from action $L$ to action $R$: $\tilde{\theta} = 0$. Otherwise, we can move at most $\varepsilon$ mass, and this is the optimal solution: $\tilde{\theta} = \theta - \varepsilon$.

Second, we solve \eqref{pb:dual}:
\[
\begin{aligned}
\lambda^\ast 
&= \argmin_{\lambda \geq 0} \lambda\varepsilon + \int_{\statespace}\int_{\actionspace} \max_{a' \in \actionspace} \left\{A^{\pi_{\theta}}(s, a') - \lambda c(a, a')\right\}\d\pi_{\theta}(a|s)\d\visitationfreq(s)\\
&= \argmin_{\lambda \geq 0} \lambda\varepsilon + \theta\max\left\{A^{\pi_{\theta}}(s_0, L), A^{\pi_{\theta}}(s_0, R) - \lambda\right\} + (1 - \theta) \max\left\{A^{\pi_{\theta}}(s_0, L) - \lambda, A^{\pi_{\theta}}(s_0, R)\right\}\\
&= \argmin_{\lambda \geq 0} \lambda\varepsilon + \theta\underbrace{\max\left\{2(\theta - 1), 2\theta - \lambda\right\}}_{=2\theta - \min\left\{2, \lambda\right\}} + (1 - \theta) \underbrace{\max\left\{2(\theta - 1) - \lambda, 2\theta\right\}}_{=2\theta}\\
&= \argmin_{\lambda \geq 0} \lambda\varepsilon - \theta\min\left\{2, \lambda\right\},
\end{aligned}
\]
which yields $\lambda^\ast = 0$ if $\theta < \varepsilon$, $\lambda^\ast = 2$ if $\theta \geq \varepsilon$. As expected, $\lambda^\ast$ is not zero if the constraint is active, which is the case every time there is more mass in $\pi(L|s_0)$ than what we are allowed to move. In practice, we can solve this optimization problem to global optimality (with sufficient numerical precision). In discrete settings, the numerical approximations introduced by the solver are the only source of approximations.

In the considered example, with finite state-action spaces, one should follow the first approach discussed for the parameter update; namely, one should update $\theta$ according to~\eqref{eq:opt-policy-discrete} and~\eqref{eq:bij*}. For the sake of providing insights and shading light on both the inner working of the proposed algorithm and its limitations, we now study all the proposed options for step 5 in \cref{algo:policy-update}.

\paragraph{Direct parametrization (finite spaces)}
We distinguish two cases: 
\begin{itemize}
    \item[$\theta > \varepsilon$.] The constraint is active: $\lambda^\ast = 2$. We obtain $\mins{\lambda^\ast}{s}{L} = \{L, R\}$ and $\mins{\lambda^\ast}{s}{R} = \{R\}$. The minimizers are not unique: we need ``mass splitting''. In particular, $\optmaplower{\lambda^*}{s_0}{L} = L$ (the ``closest'' minimizer is the one that not require any mass displacement), $\optmapupper{\lambda^*}{s_0}{L} = R$, and $\optmaplower{\lambda^*}{s_0}{R} = \optmapupper{\lambda^*}{s_0}{R} = R$. Finally, $t^\ast$ is selected to ``activate'' the constraint. In this case, $t^\ast = \varepsilon$, and we obtain $\tilde{\pi}(L|s_0) = \theta - \varepsilon \eqqcolon \tilde{\theta}$, as expected.
    \item[$\theta \leq \varepsilon$.] In this case, $\lambda^\ast = 0$. Then we have $\mins{\lambda^\ast}{s}{L} = \{R\}$ and $\mins{\lambda^\ast}{s}{R} = \{R\}$. That is, the minimizers are unique, the unique transport map $\optmaplower{\lambda^*}{s_0}{\cdot} = \optmapupper{\lambda^*}{s_0}{\cdot} = T$ is the constant map $T(L) = T(R) = R$ and we do not need ``mass splitting''. The new policy is then trivially $\tilde{\pi}(L|s_0) = \pushforward{T}{\pi}(L|s_0) = 0$; that is, $\tilde{\theta} = 0$, again as expected. 
\end{itemize}

\paragraph{Direct parametrization via policy network (continuous states, discrete actions)}
In this approach, we build on top of the result of the previous one, and we perform gradient descent on the loss $L(\theta)=\sum_{s\in\hat\statespace}\rho_{\pi_{\theta_t}}(s)\norm{\pi_\theta(\cdot|s)-\pi^\ast_{\theta_t}(\cdot|s)}^2$ to steer $\pi_{\theta}$. In particular, the minimization of $L$ steers the policy network towards the optimal policy. Hence, if we were to perform gradient descent until convergence, the policy update would be exact up to numerical precision.

\paragraph{Arbitrary policy parametrization (continuous states, continuous actions)}
Armed with $\lambda^\ast$, we construct the loss function
\[
L(\tilde{\theta}) = \sum_{s\in\hat\statespace}\int_{\actionspace}\max_{a'\in\actionspace}\{A^{\pi_{\theta}}(s,a')-\lambda^\ast c(a,a')\}\d\pi_{\tilde{\theta}}(a|s)\rho_{\pi_{\theta}}(s).
\]

In particular, $\tilde{\theta}$ affects only $\pi_{\tilde{\theta}}$. The intuition behind this loss function is to increase the probability mass where the regularized advantage function has its maxima. However, in general, the solutions are multiple, while the gradient descent will possibly converge to one. As such, one limitation of this cost function is related with the ``mass splitting'' issue. That is, if we were to perform gradient descent until convergence, the policy update might violate the trust-region constraint. Empirically, the proposed loss function performs well, and allows the deployment of optimal transport trust-region methods in continuous spaces. Moreover, the gradient descent is not iterated until convergence in practice.
Nonetheless, the arising optimization procedure deserves further study in terms of convergence and optimality guarantees, but our duality results provide an intuition on the method. Future work will focus on understanding how to implicitly describe transport \emph{plans}, rather than transport maps.

In the remainder of the discussion, we study the gradient for the running example, and we validate the intuition above: 
\[
\begin{aligned}
\nabla_{\tilde{\theta}}L(\tilde{\theta}) 
&= \nabla_{\tilde{\theta}}\left[\tilde{\theta}\max\{2(\theta - 1), 2\theta - \lambda^\ast\} + (1 - \tilde{\theta})2\theta\right]\\
&= \max\{2(\theta - 1), 2\theta - \lambda^\ast\} - 2\theta\\
&= -\min\{2, \lambda^\ast\}
\end{aligned}
\]

We consider the cases $\theta > \varepsilon$ and $\theta \leq \varepsilon$ separately: 
\begin{itemize}
    \item[$\theta > \varepsilon$.] The gradient reads $\nabla_{\tilde{\theta}}L(\tilde{\theta}) = -2$: we are steering $\theta$ to $0$. Namely, the loss is embedding the transport map $\optmapupper{\lambda^*}{s_0}{L} = R$, $\optmapupper{\lambda^*}{s_0}{R} = R$. Following the gradient corresponds to interpolating between $\pi_\theta(\cdot|s_0)$ and $\pushforward{\optmapupper{\lambda^*}{s_0}{\cdot}}{\pi_{\theta}(\cdot|s_0)}$. However, as previously discussed, the optimal solution requires ``mass splitting'': if we perform gradient descent until convergence, the policy update yields $\tilde{\theta} = 0$, which violates the trust-region constraint.
    \item[$\theta \geq \varepsilon$.] The gradient vanishes: $\nabla_{\tilde{\theta}}L(\tilde{\theta}) = -\lambda^\ast = 0$. This issue is related with the policy parametrization. The issue can be addressed with a decreasing trust-region radius $\varepsilon \to 0$. This highlights a potential drawback of the proposed practical implementation: albeit our theoretical results are valid and parametrization-independent, the choice of the network architecture affects the learning process. Aside from the possibility of a vanishing gradient, it is possible that the optimal policy is not contained in the family described by the parametrization.
\end{itemize}

\subsection{Further details on the experimental results}\label{app:details-experiments}

\subsubsection{Additional numerical details}
In addition to the experimental results of~\cref{section:experiments}, we provide the numerical results of the performance comparison among our algorithm and the baseline methods in~\cref{tab:trained:scores}. In particular, we average the expected scores obtained in the last $10\%$ of the training episodes.
\begin{table}[h]
\begin{center}
    \begin{tabular}{l r r r}
    \toprule
        Algorithm & \textit{CliffWalking-v0} & \textit{Taxi-v3} & \textit{MountainCarCont.-v0}\\\midrule
        \gls{acr:trpo}~\cite{Schulman2015} & $-568 \pm 0$ & $-51 \pm 11$ & $75 \pm 0$ \\
        \gls{acr:ppo}~\cite{Schulman2017} & $-5001 \pm 0$ & $-200 \pm 0$ & $-8 \pm 13$ \\\midrule
        A2C~\cite{Mnih2016} & $-512 \pm 0$ & $-200 \pm 0$ & $83 \pm 0$\\\midrule
        \gls{acr:wpo}~\cite{Song2022} & $-23 \pm 2$ & $-65 \pm 10$ & / \\
        \gls{acr:bgpg}~\cite{Pacchiano2020} & / & / & $90 \pm 3$\\
        \gls{acr:wnpg}~\cite{Moskovitz2020} & / & / & $88 \pm 5$ \\\midrule
        \bf \gls{acr:ottrpo} & $-14 \pm 0$ & $3 \pm 3$ & $88 \pm 6$\\\midrule
        \\
        \toprule
        Algorithm & \textit{Hopper-v3} & \textit{Swimmer-v3} & \textit{HalfCheetah-v3} \\\midrule
        \gls{acr:trpo}~\cite{Schulman2015} & $3130 \pm 230$ & $239 \pm 4$ & $2561 \pm 114$ \\
        \gls{acr:ppo}~\cite{Schulman2017} & $1388 \pm 127$ & $219 \pm 2$ & $5078 \pm 126$ \\\midrule
        A2C~\cite{Mnih2016} & $667 \pm 73$ & $24 \pm 8$ & $633 \pm 70$ \\\midrule
        \gls{acr:bgpg}~\cite{Pacchiano2020} & $1720 \pm 59$ & $37 \pm 0$ & $1309 \pm 17$\\
        \gls{acr:wnpg}~\cite{Moskovitz2020} & $1998 \pm 50$ & $40 \pm 0$ & $985 \pm 43$ \\\midrule
        \bf \gls{acr:ottrpo} & $2939 \pm 162$ & $359 \pm 2$ & $4818 \pm 4$\\\bottomrule
    \end{tabular}
    \caption{Averaged scores over last $10\%$ episodes of the training process.}\label{tab:trained:scores}
\end{center}
\end{table}

Moreover, in~\cref{tab:computationaltime} we report the training time in the Mujoco \cite{Todorov2012} environments, which are the most computationally demanding. 
The values in \cref{tab:computationaltime} have been collected under comparable conditions and similar hardware.
While the absolute values here provide little insight, the relative comparisons highlight that our algorithm has computational performances only slightly slower than \gls{acr:trpo}~\cite{Schulman2015}, \gls{acr:ppo}~\cite{Schulman2017}, and A2C~\cite{Mnih2016}. Instead, other methods levereging optiamal transport and, in particular, the Wasserstein distance (i.e., \gls{acr:bgpg}~\cite{Pacchiano2020} and \gls{acr:wnpg}~\cite{Moskovitz2020}) require a considerably longer training time.
\begin{table}[h]
    \centering
    \begin{tabular}{l r r r}
    \toprule
        Algorithm & \textit{Hopper-v3} & \textit{Swimmer-v3} & \textit{HalfCheetah-v3}\\\midrule 
        \gls{acr:trpo}~\cite{Schulman2015} & $5606 \pm 2542$ & $7882 \pm 963$ & $6719 \pm 3517$\\
        \gls{acr:ppo}~\cite{Schulman2017} & $9967 \pm 6702$ & $7875 \pm 1667$ & $10088 \pm 9966$\\
        A2C~\cite{Mnih2016} & $3190 \pm 1210$ & $7876 \pm 2332$ & $4295 \pm 172$ \\\midrule
        \gls{acr:bgpg}~\cite{Pacchiano2020} & $15009 \pm 64$ & $31402 \pm 1491$ & $50335 \pm 112$\\
        \gls{acr:wnpg}~\cite{Moskovitz2020} & $29941 \pm 12531$ & $44774 \pm 731$ & $29620 \pm 8687$ \\\midrule
        \bf \gls{acr:ottrpo} & $8031 \pm 931$ & $16903 \pm 2933$ & $11276 \pm 1410$ \\\bottomrule
    \end{tabular}
    \caption{Training time in seconds of the different algorithms in the Mujoco \cite{Todorov2012} environments.}
    \label{tab:computationaltime}
\end{table}

\subsubsection{Some considerations on convergence speed }
Finally, in some environments (e.g., \textit{Taxi-v3}, \textit{Hopper-v3}, and \textit{Swimmer-v3}), \gls{acr:trpo} seems to learn faster than the proposed algorithm (albeit not converging to a better solution). This is closely related to the choice of trust-region. With \gls{acr:ottrpo}, moving the ``probability mass'' is done at the cost $c(a_1, a_2)$. In \gls{acr:trpo}, there is no notion of transport cost in the action space: all the actions are at the same ``distance''. That is, they differ only based on the probability of using them when at a state $s \in \mathcal{\statespace}$; namely $\log{}(\pi(a_1|s) / \pi(a_2|s))$. Whenever $c(a_1, a_2) > \log{}(\pi(a_1|s) / \pi(a_2|s))$, the convergence is faster with \gls{acr:trpo} compared to \gls{acr:ottrpo}. Clearly, one could design a transport cost that speeds up the convergence (a ``smaller'' one), at the price of a potentially less stable and robust behavior of the algorithm. The study of the ``optimal'' choice of transport cost is indeed an interesting question for future research.
As an example, in the \textit{Taxi-v3} environment there are $6$ actions. The proposed method with a binary distance incurs a cost of $1$ to bring a state from a uniform probability distribution over the actions to a deterministic one (which for most states is the case in the optimal policy). Instead, \gls{acr:trpo} considers such policies only $\log{}(6) < 1$ away. Similar (but more complicated) calculations can be carried out for the Mujoco environments.
Another reason can be found in the lack of symmetry of the \gls{acr:kl} divergence. Such asymmetry can ``enforce'' a direction of exploration: going back to a previously explored point in the policy space might be more (or less) expensive. Oppositely, an optimal transport cost can be chosen to be symmetric (e.g., the Wasserstein distance and the costs used for our experiments).
When combined with errors in the estimation of the advantage function, it can result in oscillations in the policy space, which might slow down the convergence.

\section{Proofs}

\subsection{Proof of~\cref{thm:dual}}\label{appendix:strongduality}
We start with weak duality:
\begin{proposition}[Weak Duality]\label{proposition:weakduality}
Weak duality holds. Namely, $\primal \leq \dual$.
\end{proposition}
\begin{proof}
By minimax we have 
\begin{align*}
    \primal &= \sup_{\tilde\pi \in \Pi} \inf_{\lambda \geq 0} \left\{ \int_\statespace\int_\actionspace \advantage(s,a) \d\tilde\pi(a|s)d\visitationfreq(s) + \lambda \left(\varepsilon - \int_\statespace\otdiscrepancy{\pi(\cdot|s)}{\tilde\pi(\cdot|s)}\d\rho(s)\right)\right\}
    \\
    &\leq
    \inf_{\lambda \geq 0} \left\{ \lambda \varepsilon + \sup_{\tilde\pi \in \Pi}\left\{\int_\statespace\int_\actionspace \advantage(s,a) \d\tilde\pi(a|s)-\lambda\otdiscrepancy{\pi(\cdot|s)}{\tilde\pi(\cdot|s)}\d\visitationfreq(s)\right\}\right\}
    \\
    &\leq
    \inf_{\lambda \geq 0} \left\{ \lambda \varepsilon + \int_\statespace \underbrace{\sup_{\tilde \pi(\cdot|s) \in \Pi}\left\{\int_\actionspace \advantage(s,a) \d\tilde\pi(a|s)-\lambda\otdiscrepancy{\pi(\cdot|s)}{\tilde\pi(\cdot|s)}\right\}}_{(\heartsuit)}\d\visitationfreq(s)\right\}.
\end{align*}
Second, for all $s\in\statespace{}$, Kantorovich duality \cite[Theorem 6.1.1]{Ambrosio2008} gives 
\begin{align*}
    (\heartsuit)
    &=
    \sup_{\pi(\cdot|s) \in \Pi}\left\{\int_\actionspace \advantage(s,a) \d\tilde\pi(a|s)-\lambda\sup_{\phi + \psi \leq c}\left\{\int_\actionspace\phi(a) \d\pi(a|s) + \int_\actionspace\psi(a) \d\tilde\pi(a|s)\right\}\right\}.
\end{align*}
Without loss of generality, we assume $\lambda>0$; else, the result is straightforward. Then, we choose
\begin{equation*}
\psi(\cdot) = \advantage(s,\cdot) / \lambda
\quad 
\text{and}
\quad 
\phi(\cdot) =\inf_{a'\in\actionspace}\left\{c(\cdot,a')-\psi(a')\right\}.
\end{equation*}
By \cref{hyp:state-action-space,hyp:continuous-advantage} $\phi,\psi$ are continuous on a compact space. Hence, they are also bounded \cite[Theorem 27.4]{Munkres2000}. Moreover, 
\begin{equation*}
    \phi(a_1)+\psi(a_2)
    =
    \inf_{a'\in\actionspace}\{c(a_1,a')-\psi(a')\} + \psi(a_2)
    \leq
    c(a_1,a_2) -\psi(a_2) + \psi(a_2)\leq c(a_1,a_2).
\end{equation*}
Thus, $(\phi,\psi)$ is a valid, possibly sub-optimal, choices for the supremum.
Overall, we get the upper bound
\begin{align*}
    (\heartsuit)
    &
    \leq
    \sup_{\pi(\cdot|s) \in \Pi}
    \left\{\int_{\actionspace}-\inf_{a'\in\actionspace}\left\{\lambda c(a,a')-\advantage(s,a')\right\}\d\pi(a|s)\right\}
    \\
    &=
    \int_{\actionspace}\sup_{a'\in\actionspace}\left\{\advantage(s,a')-\lambda c(a,a')\right\}\d\pi(a|s).
\end{align*}
Hence, we obtain
\begin{equation*}
    \primal
    \leq
    \inf_{\lambda \geq 0}\left\{ \lambda \varepsilon +\int_\statespace\int_{\actionspace}\sup_{a'\in\actionspace}\left\{\advantage(s,a')-\lambda c(a,a')\right\}\d\pi(a|s)\d\visitationfreq(s)\right\} = \dual.
    \qedhere 
\end{equation*}
\end{proof}

The proof of strong duality (i.e., equality) is more delicate, and requires some preliminary results. First, we recall the definition the regularization operator $\regularizer_\lambda$, which represents intuitively a ``regularized'' advantage function: 
\begin{equation*}
\begin{aligned}
    \regularizer_\lambda: \statespace\times\actionspace &\to\reals\\
    (s,a) &\mapsto\sup_{a'\in\actionspace}\left\{\advantage{}(s,a')-\lambda c(a,a')\right\}.
\end{aligned}
\end{equation*}
The regularization operator $\regularizer_\lambda$ is well defined by \cref{hyp:continuous-advantage,hyp:state-action-space}. Indeed, for all $\lambda \in \reals, s \in \statespace, a \in \actionspace$, $a' \mapsto \advantage{}(s,a')-\lambda c(a,a')$ is a continuous function on a compact space and thus it attains its maximum value (e.g., see \cite[Theorem 27.4]{Munkres2000}). Accordingly, we recall the set of maximizers
\[
\mins{\lambda}{s}{a}\coloneqq\argmax_{a' \in \actionspace}\left\{\advantage{}(s,a')-\lambda c(a,a')\right\}.
\]
Since minimizers are generally \emph{not} unique, $\mins{\lambda}{s}{a}$ is indeed a set. In fact, it is also closed: 

\begin{lemma}\label{proposition:minimizers:attained}
For all $\lambda \in \reals, s \in\statespace, a \in\actionspace$, the set $\mins{\lambda}{s}{a}\subset\actionspace$ is non-empty and closed (and thus measurable).
\end{lemma}
\begin{proof}
Since $a' \mapsto \advantage{}(s,a')-\lambda c(a,a')$ is a continuous function on a compact space and thus it attains its maximum value (e.g., see \cite[Theorem 27.4]{Munkres2000}), which directly implies that $\mins{\lambda}{s}{a}$ is non-empty.
To prove that it is closed, it suffices to show that $\mins{\lambda}{s}{a}$ contains all its limit points (e.g., see~\cite[Corollary 17.7]{Munkres2000}). Let $(a_n)_{n \in \naturals}, a_n \in \mins{\lambda}{s}{a}$ be a sequence converging to $\bar{a}\in\actionspace$.
Let $h^{\lambda}_{s,a}: \actionspace \to \reals$ be defined as $h^{\lambda}_s(a) =\advantage{}(s,a')-\lambda c(a,a')$. Since $a_n \in \mins{\lambda}{s}{a}$, $h^{\lambda}_{s,a}(a_n) = \regularizer_{\lambda}(s,a)$ for all $n \in \naturals$. Moreover, $\advantage$ is continuous by \cref{hyp:continuous-advantage}, and thus so is $h$. By continuity,
\[
h^{\lambda}_{s,a}(\bar{a}) = \lim_{n\to\infty}h^{\lambda}_{s,a}(a_n) = \regularizer_{\lambda}(s,a),
\]
and so $\bar{a} \in \mins{\lambda}{s}{a}$.
Measurability of $\mins{\lambda}{s}{a}$ follows from closedness and $\pi(\cdot|s)$ being a Borel probability measure. 
\end{proof}

Among all possible minimizers, the \emph{closest} and \emph{furthest apart} will play an important role in the proof of strong duality. Thus, we define
\[
    \minsdistupper{\lambda}{s}{a}
    \coloneqq
    \max_{a' \in \mins{\lambda}{s}{a}} c(a,a'),
    \qquad
    \minsdistlower{\lambda}{s}{a}
    \coloneqq 
    \min_{a' \in \mins{\lambda}{s}{a}}c(a,a'),
\]
where the $\max$ and $\min$ are well defined by \cref{proposition:minimizers:attained}, since $c(a,\cdot)$ is a continuous function for all $a\in\actionspace$ and $\mins{\lambda}{s}{a}$ is compact (being the closed subset of a compact space, \cite[Theorem 26.2]{Munkres2000}).

These definitions allow us to study the properties of $\Phi$ more in detail:

    

\begin{lemma}[Properties of $\regularizer$]\label{lemma:regularizationoperator:properties}
The regularization operator $\regularizer$ has the following properties:
\begin{enumerate}[leftmargin=*]
    \item It is uniformly (in $s$, $a$, and $\lambda$) lower and upper bounded by some constant.
    \item It is lower semi-continuous, non-increasing, and convex in $\lambda$.
    \item For all $\lambda>0$ the left and right derivatives are
    \[
        \pdv{\regularizer_{\lambda}(s,a)}{\lambda-} = \minsdistupper{\lambda}{s}{a},
        \qquad
        \pdv{\regularizer_{\lambda}(s,a)}{\lambda+}=\minsdistlower{\lambda}{s}{a}.
    \]
    \item There are measurable selections
    \[
    \optmapupper{\lambda}{s}{a} \in \argmax_{a' \in \mins{\lambda}{s}{a}}c(a,a'), 
    \qquad
    \optmaplower{\lambda}{s}{a} \in \argmin_{a' \in \mins{\lambda}{s}{a}}c(a,a').
    \]
\end{enumerate}
\end{lemma}
\begin{proof}
The proof follows closely \cite[Lemma 3]{Gao2016}. Specifically, we prove the steps separately:
\begin{enumerate}[leftmargin=*]
    \item Since $\advantage$ is continuous on a compact space, its absolute value is bounded by some $C$ (e.g., see \cite[Theorem 27.4]{Munkres2000}).
    For the upper bound: $\regularizer_\lambda$ results from an infimum, $\lambda c(a,a') \geq 0$, and $a' = a$ is always a valid choice, yielding $\regularizer_\lambda(s,a) \geq \advantage(s,a)-\lambda c(a,a)\geq -C$
    For the lower bound; since transport costs are non-negative, we also have $\regularizer_\lambda(s,a)\leq \sup_{a'\in\actionspace}\advantage(s,a')\leq C$. Thus, $|\regularizer_\lambda(s,a)|\leq C$ for all $\lambda\geq 0$, $s\in\statespace$, and $a\in\actionspace$.
    
    \item 
    For all $s \in \statespace{}, a \in \actionspace{}$, $\lambda \mapsto \regularizer_\lambda(s,a)$ is the supremum of non-increasing and affine functions. As a result, it is non-increasing, convex and lower semi-continuous in $\lambda$.

    \item The proof boils down to upper semi-continuity and lower semi-continuity of $\lambda\mapsto\minsdistupper{\lambda}{s}{a}$ and $\lambda\mapsto\minsdistlower{\lambda}{s}{a}$, respectively, whose proof is identical to the setting of \gls{acr:dro} since both $s$ and $a$ are frozen; see~\cite[Corollary 1]{Gao2016} and replace the distance $d^p$ with $c$. Then, the results follows from~\cite[Lemma 4.iii]{Gao2016}, since, by definition, the growth rate is 0 for compact spaces and $\regularizer>-\infty$.
        
    \item The set of minimizers $\mins{\lambda}{s}{a}$ is non-empty and closed by~\cref{proposition:minimizers:attained}. Since it is a closed subset of a compact space, it is also compact~\cite[Theorem 26.2]{Munkres2000}. Moreover, $c(a, \cdot)$ is continuous for all $a\in\actionspace$ and $c(\cdot, a')$ is continuous (and thus measurable) for all $a'\in\actionspace$, thus the result follows from \cite[Theorem 18.19]{Aliprantis2006}. 
    \qedhere 
\end{enumerate}
\end{proof}

We are now ready to prove strong duality:

\begin{proof}[Proof of \cref{thm:dual}]
In view of \cref{proposition:weakduality}, we only need to prove $\primal \geq \dual$.
We split the proof in two steps. First, we show existence of an optimal multiplier $\lambda^\ast \in [0,\infty)$. Second, we study the case $\lambda^\ast > 0$, whereas $\lambda^\ast = 0$ is straightforward from direct inspection of $\primal$ and $\dual$.

Let $h: \nonnegativeReals \to \reals$ be defined as
\[
h(\lambda) = \lambda\varepsilon + \int_\statespace\int_\actionspace \regularizer_{\lambda}(s,a)\d\pi(a|s)d\visitationfreq(s).
\]
By \cref{lemma:regularizationoperator:properties}.1-2, $h$ is convex and lower semi-continuous in $\lambda$. Moreover, $h$ is coercive:
\begin{equation*}
\begin{aligned}
    \liminf_{\lambda \to \infty}h(\lambda)
    &=
    \liminf_{\lambda \to \infty}\lambda\varepsilon
    +\int_\statespace\int_\actionspace\regularizer_{\lambda}(s,a)\d\pi(a|s)d\visitationfreq(s)
    \\
    &\geq 
    \lim_{\lambda \to \infty}\lambda\varepsilon
    +\int_{\statespace}\int_{\actionspace}\advantage(s,a)\d\pi(s,a)\d\visitationfreq(s)
    \\
    &=+\infty,
\end{aligned}
\end{equation*}
since $\regularizer_\lambda(s,a)\geq \advantage(s,a)$ for all $a\in\actionspace$ and $s\in\statespace$.
Thus, there exists $\lambda^\ast \in \argmin_{\lambda \geq 0} h(\lambda)$. Indeed, by definition of limit, there exists $\bar{\lambda}$ such that $h(\lambda) > h(0)$ for all $\lambda > \bar{\lambda}$. Hence, when looking for a minimum, we can restrict ourselves to $\lambda \in [0, \bar{\lambda}]$. Since $h$ is lower semi-continuous on a compact space~\cite[Corollary 27.2]{Munkres2000}, it attains its minimum at some $\lambda^\ast \in [0, \bar{\lambda}]$~\cite[Theorem 27.4]{Munkres2000}.

Assume now that $\lambda^\ast > 0$. We study the first-order optimality conditions on $h$. Since $h$ is generally non-smooth (due to the maximum in $\regularizer_{\lambda}$), we express them via the left and right derivatives, well-defined for convex functions~\cite[Theorem 23.1]{Rockafellar2015}, as $\pdv{h(\lambda)}{\lambda-}\leq 0$ and $\pdv{h(\lambda)}{\lambda+}\geq 0$.

Hence, we have
\begin{equation}\label{eq:proofstrongdualityfirstorder:one}
\begin{aligned}
    \pdv{h(\lambda)}{\lambda-}\leq 0 \iff \varepsilon
    &\leq
    \pdv{}{\lambda-}\left(\int_\statespace\int_\actionspace \regularizer_{\lambda^\ast}(s,a)\d\pi(a|s)\d\visitationfreq(s)\right)
    \\
    \overset{\heartsuit}&{=}
    \int_\statespace\int_\actionspace \pdv{}{\lambda-}\regularizer_{\lambda^\ast}(s,a)\d\pi(a|s)\d\visitationfreq(s)
    \\
    \overset{\diamondsuit}&{=}
    \int_\statespace\int_\actionspace \minsdistupper{\lambda^\ast}{s}{a}\d\pi(a|s)\d\visitationfreq(s)
    \\
    \overset{\triangle}&{=}
    \int_\statespace\int_\actionspace c(a,\optmapupper{\lambda^\ast}{s}{a})\d\pi(a|s)\d\visitationfreq(s),
\end{aligned}
\end{equation}
and 
\begin{equation}\label{eq:proofstrongdualityfirstorder:two}
\begin{aligned}
    \pdv{h(\lambda)}{\lambda+}\geq 0 \iff \varepsilon &\geq \pdv{}{\lambda+}\left(\int_\statespace\int_\actionspace \regularizer_{\lambda^\ast}(s,a)\d\pi(a|s)\d\visitationfreq(s)\right)
    \\ 
    \overset{\heartsuit}&{=}
    \int_\statespace\int_\actionspace \pdv{}{\lambda+}\regularizer_{\lambda^\ast}(s,a)\d\pi(a|s)\d\visitationfreq(s)
    \\
    \overset{\diamondsuit}&{=}\int_\statespace\int_\actionspace \minsdistlower{\lambda^\ast}{s}{a}d\pi(a|s)d\visitationfreq(s)
    \\
    \overset{\triangle}&{=} \int_\statespace\int_\actionspace c(a,\optmaplower{\lambda^\ast}{s}{a})\d\pi(a|s)\d\visitationfreq(s),
\end{aligned}
\end{equation}
where in
\begin{itemize}[leftmargin=*]
    \item[$\heartsuit$] we use that, for all $\lambda$, $\minsdistupper{\lambda}{s}{a} = c(a,\optmapupper{\lambda}{s}{a})$, $\minsdistlower{\lambda}{s}{a} = c(a,\optmaplower{\lambda}{s}{a})$ result from the composition of a continuous (the cost) and a measurable function ($\optmapupper{\lambda}{s}{a}$ and $\optmaplower{\lambda}{s}{a}$), and thus they are measurable~\cite[Theorem 1.8]{Rudin1987}. Integrability follows directly from non-negativity~\cite[Theorem 1.17, Definition 1.23]{Rudin1987}. Hence, we can use the differentiation lemma \cite[Theorem 6.28]{Klenke2008} together with \cref{lemma:regularizationoperator:properties};
    
    \item[$\diamondsuit$] we use \cref{lemma:regularizationoperator:properties}.3; and
    
    \item[$\triangle$] we use the definition of $\optmaplowername{\lambda^\ast}$, $\optmapuppername{\lambda^\ast}$ and of $\minsdistlower{\lambda^\ast}{s}{a}$, $\minsdistupper{\lambda^\ast}{s}{a}$.
    
\end{itemize}
Overall, the inequalities in~\eqref{eq:proofstrongdualityfirstorder:one} and~\eqref{eq:proofstrongdualityfirstorder:two} lead to 
\begin{equation}\label{eq:proofstrongdualityinequalitiescombined}
\int_\statespace\int_\actionspace c(a,\optmaplower{\lambda^\ast}{s}{a})\d\pi(a|s)\d\visitationfreq(s) \leq
\varepsilon
\leq
\int_\statespace\int_\actionspace c(a,\optmapupper{\lambda^\ast}{s}{a})\d\pi(a|s)\d\visitationfreq(s).
\end{equation}
With this observation, we can construct a primal solution $\tilde\pi(\cdot|s)$ and show its optimality. 
First, by \eqref{eq:proofstrongdualityinequalitiescombined}, there exists $t^\ast \in [0,1]$ such that
\[
t^\ast\int_\statespace\int_\actionspace c(a,\optmaplower{\lambda^\ast}{s}{a})\d\pi(a|s)\d\visitationfreq(s)
+
(1-t^\ast) \int_\statespace\int_\actionspace c(a,\optmapupper{\lambda^\ast}{s}{a})\d\pi(a|s)\d\visitationfreq(s) = \varepsilon.
\]
For all $s\in\statespace{}$ consider the policy
\[
\tilde\pi(\cdot|s)\coloneqq t^\ast\pushforward{\optmaplower{\lambda^\ast}{s}{\cdot}}{\pi(\cdot|s)} + (1 - t^\ast)\pushforward{\optmapupper{\lambda^\ast}{s}{\cdot}}{\pi(\cdot|s)}.
\]
By construction, $\pi(\cdot|s)$ is a feasible solution. Indeed, with $\Id$ being the identity mapping $\Id(a) = a$, we can consider the (possibly sub-optimal) transport plan $\gamma_s \in \coupling{\pi(\cdot|s)}{\tilde\pi(\cdot|s)}$, defined as
\[
\gamma_s\coloneqq
t^\ast \pushforward{(\Id, \optmaplower{\lambda^\ast}{s}{\cdot})}{\pi(\cdot|s)}
+
(1-t^\ast) \pushforward{(\Id, \optmapupper{\lambda^\ast}{s}{\cdot}}{\pi(\cdot|s)}.
\]
Then, by definition of optimal transport discrepancy and monotonicity of the integral we have
\begin{equation}\label{eq:proofstrongdualityfeasible}
\begin{aligned}
    \int_\statespace\otdiscrepancy{\pi(\cdot|s)}{\tilde\pi(\cdot|s)}\d\visitationfreq(s)
    &\leq \int_\statespace\int_{\actionspace\times\actionspace} c(a,a') \d\gamma_s(a,a')\d\visitationfreq(s)\\
    &=t^\ast\int_\statespace\int_\actionspace c(a,\optmaplower{\lambda^\ast}{s}{a})\d\pi(a|s)\d\visitationfreq(s)\\
    &~~+(1-t^\ast)\int_\statespace\int_\actionspace c(a,\optmapupper{\lambda^\ast}{s}{a})\d\pi(a|s)\d\visitationfreq(s)
    \\
    &= \varepsilon.
\end{aligned}
\end{equation}
To conclude, by definition, we have 
\begin{equation}\label{eq:proofstrongdualityadvantage}
\begin{split}
\advantage(s,\optmapupper{\lambda^\ast}{s}{a})
&=
\regularizer_{\lambda^\ast}(s,a) + \lambda^\ast c(a,\optmapupper{\lambda^\ast}{s}{a})
\\
\advantage(s,\optmaplower{\lambda^\ast}{s}{a})
&=
\regularizer_{\lambda^\ast}(s,a) + \lambda^\ast c(a,\optmaplower{\lambda^\ast}{s}{a})
\end{split}
\end{equation}
and thus
\[
\begin{aligned}
\primal &\geq \int_\statespace\int_\actionspace \advantage(s,a) \d\tilde\pi(a|s)\d\visitationfreq(s)\\
&=
\int_\statespace\int_\actionspace\left( t^\ast\advantage(s,\optmaplower{\lambda^\ast}{s}{a}) + (1-t^\ast)\advantage(s,\optmapupper{\lambda^\ast}{s}{a}) \right)\d\pi(a|s)\d\visitationfreq(s)\\
&=
t^\ast\int_\statespace\int_\actionspace \advantage(s,\optmaplower{\lambda^\ast}{s}{a}) \d\pi(a|s)\d\visitationfreq(s)
+ (1-t^\ast)\int_\statespace\int_\actionspace \advantage(s,\optmapupper{\lambda^\ast}{s}{a}) \d\pi(a|s)\d\visitationfreq(s)\\
\overset{\eqref{eq:proofstrongdualityadvantage}}&{=}\lambda^\ast\int_\statespace \left(t^\ast c(a,\optmaplower{\lambda^\ast}{s}{a}) + (1 - t^\ast)c(a,\optmapupper{\lambda^\ast}{s}{a})\right)\d\visitationfreq(s) \\
&~~+ \int_\statespace\int_\actionspace \regularizer_{\lambda^\ast}(s,a) \d\pi(a|s)\d\visitationfreq(s)
\\
\overset{\eqref{eq:proofstrongdualityfeasible}}&{=}\lambda^\ast\varepsilon + \int_\statespace\int_\actionspace \regularizer_{\lambda^\ast}(s,a) \d\pi(a|s)\d\visitationfreq(s)
\\
&=\dual,
\end{aligned}
\]
proving $\primal \geq \dual$. In particular, we have $\primal = \dual$ and $\pi$ is the primal solution. Thus, the supremum over $\Pi$ is attained.
\end{proof}

The proofs of~\cref{cor:dual-discrete,cor:optimal-policy} then follow directly. 
\begin{remark}
Our theoretic results readily extends to transport costs $c_{s,\pi}$ parametrized by the state $s$ or the policy $\pi$, as long as the mapping $(s,\pi,a_1,a_2)\mapsto c_{s,\pi}(a_1,a_2)$ is continuous.
\end{remark}

A natural question is whether the results of \cref{thm:dual} can be extended to Sinkhorn divergence (or entropy-regularized optimal transport discrepancies in general), in particular given the success of these in the context of ~\gls{acr:rl} \cite{Cuturi2013,Peyre2019}. Our duality results do not directly apply to Sinkhorn divergence (or entropy-regularized optimal transport discrepancies in general). For instance, our proof of weak duality (\cref{proposition:weakduality}) leverages Kantorovich duality, which is optimal transport specific. However, duality results for entropy-regularized optimal transport \cite{Terjek2021}, and recent results in \gls*{acr:dro} \cite{Gao2017} suggest that our proof can be adapted to such constraints. Having said this, we remark that Sinkhorn divergence was successful in mitigating the burden of optimal transport computations. The proposed algorithm does not rely on any of such, and it is thus unclear if regularized optimal transport discrepancies yield any benefit in this context.

\subsection{Proof of~\cref{prop:exact-policy-improv}}
\label{appendix:proof-policy-improvement}

\begin{proof}
First, recall the performance difference lemma~\cite[Lemma 6.1]{Kakade2002b}:
\begin{equation*}
    J(\tilde\pi^\ast)-J(\pi)
    =
    \frac{1}{1-\gamma}\int_{\statespace{}}\int_{\actionspace{}} \advantage(s,a)\d\tilde\pi^\ast(a|s)\d\visitationfreqNew(s).
\end{equation*}
Second, recall that from~\cref{cor:optimal-policy} that an optimal policy update is
\begin{equation}\label{eq:improvementoptimalpolicy}
    \tilde\pi^\ast(\cdot|s)\coloneqq t^\ast\pushforward{\optmaplower{\lambda^\ast}{s}{\cdot}}{\pi(\cdot|s)} + (1 - t^\ast)\pushforward{\optmapupper{\lambda^\ast}{s}{\cdot}}{\pi(\cdot|s)}.
\end{equation}
Then, the proof follows from the properties of $\tilde\pi^\ast$:\allowdisplaybreaks
\begin{align*}
J(\tilde\pi^\ast)-J(\pi)
&=
\frac{1}{1-\gamma}\int_{\statespace{}}\int_{\actionspace{}} \advantage(s,a)\d\tilde\pi^\ast(a|s)\d\visitationfreqNew(s)
\\
\overset{\eqref{eq:improvementoptimalpolicy}}&{=}
\frac{1}{1-\gamma}\int_{\statespace{}}\int_{\actionspace{}} t^\ast\advantage(s,\optmaplower{\lambda^\ast}{s}{a}) + (1-t^\ast) \advantage(s,\optmapupper{\lambda^\ast}{s}{a})\d\pi(a|s)\d\visitationfreqNew(s)
\\
\overset{\heartsuit}&{\geq} 
\begin{aligned}[t]
&\frac{1}{1-\gamma}\int_{\statespace{}}\int_{\actionspace{}} t^\ast\advantagehat(s,\optmaplower{\lambda^\ast}{s}{a}) + (1-t^\ast) \advantagehat(s,\optmapupper{\lambda^\ast}{s}{a})\d\pi(a|s)\d\visitationfreqNew(s)
\\
&-\frac{\error}{1-\gamma}
\end{aligned}
\\
\overset{\diamondsuit}&{=}
\begin{aligned}[t]
&\frac{\lambda^\ast}{1-\gamma}\int_{\statespace{}}\int_{\actionspace{}}
t^\ast c(a,\optmaplower{\lambda^\ast}{s}{a}) + 
(1-t^\ast)c(a,\optmapupper{\lambda^\ast}{s}{a})
\d\pi(a|s)\d\visitationfreqNew(s)
\\
&
+\frac{1}{1-\gamma}\int_{\statespace{}}\int_{\actionspace{}}
\regularizer_{\lambda^\ast}(s,a)\d\pi(a|s)\d\visitationfreqNew(s)
-\frac{\error}{1-\gamma}
\end{aligned}
\\
\overset{\triangle}&{\geq}
\begin{aligned}[t]
&\frac{\lambda^\ast}{1-\gamma}\int_{\statespace{}}
\otdiscrepancy{\pi(\cdot|s)}{\tilde\pi^\ast(\cdot|s)}
\d\visitationfreqNew(s)
\\
&+\frac{1}{1-\gamma}\int_{\statespace{}}\int_{\actionspace{}}
\advantagehat(s,a)\d\pi(a|s)\d\visitationfreqNew(s)
-\frac{\error}{1-\gamma}
\end{aligned}
\\
\overset{\heartsuit}&{\geq}
\begin{aligned}[t]
&\frac{\lambda^\ast}{1-\gamma}\int_{\statespace{}}
\otdiscrepancy{\pi(\cdot|s)}{\tilde\pi^\ast(\cdot|s)}
\d\visitationfreqNew(s)
\\
&+\frac{1}{1-\gamma}\int_{\statespace{}}\int_{\actionspace{}}
\advantage(s,a)\d\pi(a|s)\d\visitationfreqNew(s)
-\frac{2\error}{1-\gamma}
\end{aligned}
\\
\overset{\square}&{=} 
\frac{\lambda^\ast}{1-\gamma}\int_{\statespace{}}
\otdiscrepancy{\pi(\cdot|s)}{\tilde\pi^\ast(\cdot|s)}
\d\visitationfreqNew(s)
-\frac{2\error}{1-\gamma},
\end{align*}
where in
\begin{itemize}[leftmargin=*]
    \item[$\heartsuit$] we use that, by assumption, $\advantage(s,a)\geq \advantagehat(s,a)-\error$ and $\advantagehat(s,a)\geq \advantage(s,a)-\error$ for all $s\in\statespace{}$ and $a\in\actionspace{}$;
    
    \item[$\diamondsuit$] by definition of the regularization operator $\regularizer$ and of $\optmaplower{\lambda^\ast}{s}{a}$ and $\optmapupper{\lambda^\ast}{s}{a}$ (\cref{lemma:regularizationoperator:properties}.4), for all $s\in\statespace{}$ and $a\in\actionspace{}$ it holds
    \begin{equation*}
        \advantage(s,\optmapupper{\lambda^\ast}{s}{a})
         =
        \regularizer_{\lambda^\ast}(s,a)+\lambda^\ast c(a,\optmapupper{\lambda^\ast}{s}{a}).
    \end{equation*}
    
    \item[$\triangle$] as in \eqref{eq:proofstrongdualityfeasible}, it holds
    \begin{equation*}
        C(\pi^\ast(\cdot|s), \pi(\cdot|s))
        \leq 
        \int_{\actionspace{}}
        t^\ast c(a,\optmaplower{\lambda^\ast}{s}{a}) + 
        (1-t^\ast)c(a,\optmapupper{\lambda^\ast}{s}{a})
        \d\pi(a|s)
    \end{equation*} 
    and $\advantage(s,a)\leq \regularizer_{\lambda^\ast}(s,a)$, since for all $\lambda\in\reals, a\in\actionspace{}$, and $s\in\statespace{}$
    \begin{equation*}
    \begin{aligned}
        \regularizer_{\lambda}(s,a)
        &= 
        \max_{a'\in\actionspace{}}\left\{\advantage(s,a')-\lambda c(a,a')\right\}
        \\
        &\geq \advantage(s,a) - c(a,a)
        \\
        &=\advantage(s,a);
    \end{aligned}
    \end{equation*}

    \item[$\square$] the expected value of the advantage function vanishes, since, by definition, for all $s\in\statespace$
    \begin{equation*}
        \int_{\actionspace{}}\advantage(s,a)\d\pi(a|s)
        =
        \int_{\actionspace{}}Q^\pi(s,a)-V^\pi(s)\d\pi(a|s)
        =
        \int_{\actionspace{}}Q^\pi(s,a)\d\pi(a|s)-V^\pi(s)
        =
        0.
    \end{equation*}
\end{itemize}
This concludes the proof. 
\end{proof}

\end{document}